\documentclass[]{fairmeta}

\usepackage[utf8]{inputenc}
\usepackage[T1]{fontenc}
\usepackage{amsmath,amssymb,amsfonts,amsthm}
\usepackage{mathtools}

\DeclareMathOperator*{\argmax}{arg\,max}
\DeclareMathOperator*{\argmin}{arg\,min}

\definecolor{maroon}{RGB}{128,0,0}

\newcommand{\CPIRR}{\textsc{CPI-RR}}
\newcommand{\CPICARO}{\textsc{CPI-CARO}}
\newcommand{\SRPO}{\textsc{SRPO}}
\newcommand{\RRPO}{\textsc{RRPO}}
\newcommand{\Algref}[1]{Algorithm~\ref{#1}}

\newtheorem{theorem}{Theorem}
\newtheorem{lemma}[theorem]{Lemma}
\newtheorem{proposition}[theorem]{Proposition}
\newtheorem{corollary}[theorem]{Corollary}
\usepackage{graphicx}
\usepackage{booktabs}
\usepackage{algorithm}
\usepackage{algpseudocode}
\usepackage{enumitem}
\usepackage{hyperref}
\usepackage{natbib}
\usepackage[most]{tcolorbox}
\usepackage{tikz}
\usetikzlibrary{arrows.meta}

\definecolor{systembg}{HTML}{F0F0F0}
\definecolor{systemborder}{HTML}{999999}
\definecolor{userbg}{HTML}{FFFFFF}
\definecolor{userborder}{HTML}{4A90D9}
\definecolor{asstbg}{HTML}{F8F8FF}
\definecolor{asstborder}{HTML}{999999}

\newtcolorbox{systemprompt}[1][]{
  colback=systembg, colframe=systemborder,
  fonttitle=\bfseries\small, title=System,
  boxrule=0.5pt, arc=2pt, left=4pt, right=4pt, top=2pt, bottom=2pt,
  #1
}

\newtcolorbox{promptpanel}[1][]{
  colback=userbg, colframe=userborder,
  fonttitle=\bfseries\small,
  boxrule=0.5pt, arc=2pt, left=4pt, right=4pt, top=2pt, bottom=2pt,
  #1
}

\title{Credit Assignment with Resets in Language Model Reasoning }

\author[1,2,*]{Ankur Samanta}
\author[1]{Akshayaa Magesh}
\author[1]{Ayush Jain}
\author[1]{Youliang Yu}
\author[1]{Daniel Jiang}
\author[1]{Kavosh Asadi}
\author[3]{Kaveh Hassani}
\author[2]{Paul Sajda}
\author[1]{Jalaj Bhandari}
\author[4]{Yonathan Efroni}

\affiliation[1]{Meta AI}
\affiliation[2]{Columbia University}
\affiliation[3]{Meta Superintelligence Labs}
\affiliation[4]{Tel Aviv University}

\contribution[*]{Work done at Meta}

\abstract{

Contemporary reinforcement learning with verifiable reward methods post-train language models on multi-step reasoning by assigning a single outcome reward uniformly across all tokens in a trajectory. Such uniform assignment ignores which steps contributed to success or failure. Improving credit assignment can address this limitation by enabling targeted refinement of faulty reasoning steps, rather than updating entire trajectories uniformly. Resets are one such simple mechanism, enabling more precise credit assignment by returning to an intermediate state and resampling counterfactual continuations, so that outcome differences can be attributed to decisions made at that point. We propose two such methods: Random-Reset Policy Optimization (\RRPO{}), where reset states are drawn randomly from reasoning steps, and Self-Reset Policy Optimization (\SRPO{}), where the model self-localizes the erroneous step in an incorrect trajectory and resets there. We analyze these methods within the Conservative Policy Iteration (CPI) framework. Extending CPI with a credit-assignment oracle that targets improvable states yields provable improvements over random resets. Across models and reasoning benchmarks, \SRPO{} consistently outperforms standard GRPO and \RRPO{} by sampling multiple suffix continuations at a self-localized reset and learning from their rewards, using only the model itself with no external supervision.}

\date{\today}
\correspondence{Ankur Samanta at \email{as7416@columbia.edu}}

\begin{document}

\maketitle

\section{Introduction}
\label{sec:introduction}

Contemporary reinforcement learning with verifiable rewards (RLVR) methods post-train language models by propagating a single outcome reward uniformly across all intermediate steps. For language models on long multi-step reasoning tasks, this assigns identical credit to every step, obscuring which were most responsible for the observed outcome. Learning from steps that most directly shape the outcome can enable efficient learning as it provides a more meaningful signal for policy updates. Such mechanisms are also found in biological systems, which reinstate high-impact decision points and learn from counterfactual ``what if'' outcomes simulated from those states~\citep{witkowski2024credit,gerstenberg2024counterfactual}. Motivated by this, we study RL post-training methods that use resets to improve credit assignment by learning from the counterfactual outcomes of improvable states.

\begin{tcolorbox}[enhanced, colback=orange!5, colframe=orange!70, arc=4pt, boxrule=1pt, left=10pt, right=10pt, top=8pt, bottom=8pt, title={Credit Assignment with Resets}, fonttitle=\bfseries, coltitle=white, colbacktitle=orange!70, attach boxed title to top left={xshift=10pt, yshift=-8pt}, boxed title style={colback=orange!70, colframe=orange!70, arc=2pt, boxrule=0pt}]
A reset re-enters a previously-visited state and resamples a continuation, binding outcome differences to the decisions made from that state. In this work, we argue that resets are useful for credit assignment when the reset state admits a strictly better action, namely, when it has significant potential for improvement.
\end{tcolorbox}

We propose two reset-based RLVR post-training methods for language models, Random-Reset Policy Optimization (\RRPO{}) and Self-Reset Policy Optimization (\SRPO{}). Both resample multiple suffix continuations from a reset state drawn from a failed trajectory, and apply the policy gradient only to those suffix tokens. This shared-prefix group can be thought of as a counterfactual group of rollouts whose outcome differences attribute credit to the divergent suffix tokens. They differ in how the reset state is chosen. \RRPO{} chooses it uniformly at random across the reasoning steps, while \SRPO{} chooses it via self-localization of the first erroneous step. \SRPO{} requires no external step-level feedback, leveraging \citet{samanta2026structure}'s finding that language models can self-localize the first erroneous thought in failed structured-reasoning traces well enough to drive self-correction (see Figure~\ref{fig:srpo_overview}).

To quantify the gain over random resets from self-localization of errors, we use the framework of Conservative Policy Iteration (CPI;~\citealp{kakade2002cpi}), a foundational algorithm for policy optimization. CPI draws on-policy state samples via random resets, estimates advantages, and applies a conservative policy update. We refer to this algorithm as CPI-with-random-resets (\CPIRR{}). We compare its performance to an alternative CPI algorithm that assumes access to a credit-assignment oracle: a membership test for improvable states, answering whether a state admits an action with advantage greater than a threshold $\tau$. The resulting variant, \CPICARO{}, uses the oracle to draw reset states only from improvable states and applies the policy update only there. We establish that \CPICARO{} reduces sample complexity by $1/p_\pi^2$ and increases per-iteration improvement by $1/p_\pi$ over \CPIRR{}, where $p_\pi$ is the on-policy probability of reaching improvable states.

On a 10-benchmark suite spanning math, science, strategic, and commonsense reasoning, \SRPO{} outperforms GRPO and contemporary RL baselines that use self-correction or shared-prefix continuation. We also test \SRPO{} in a coding domain (LiveCodeBench), where it converges to a higher pass rate and learns 2--3$\times$ faster than GRPO and \RRPO{}. Higher-quality self-localizations yield higher correction rates and better suffix groups: clean prefixes correct nearly 2$\times$ as often as erroneous ones, establishing explicit self-localization as an \emph{imperfect but effective} proxy for the credit-assignment oracle. This sensitivity to localization quality motivates further work on improving it in reset-based RL to enable more efficient learning.

\paragraph{Contributions.}
\begin{enumerate}
  \item \emph{Conceptual.} We frame resets as a credit-assignment primitive for post-training of language models: outcome differences across suffixes resampled from an improvable state bind credit to the decisions made at that state.
  \item \emph{Theory.} We extend CPI with a credit-assignment oracle that concentrates resets on states with room to improve, yielding provable improvements over random resets (Theorem~\ref{thm:main}).
  \item \emph{Algorithm and empirical.} We propose \SRPO{}, a post-training RLVR method that resets to a self-localized erroneous step and resamples several alternative continuations, requiring no external step-level supervision. \SRPO{} shows improved performance compared to GRPO. We further provide guidelines for designing reset-based RL methods, and show that self-localization quality drives correction rate and reset-group quality.
\end{enumerate}

\begin{figure}[!t]
\centering
\includegraphics[width=\textwidth]{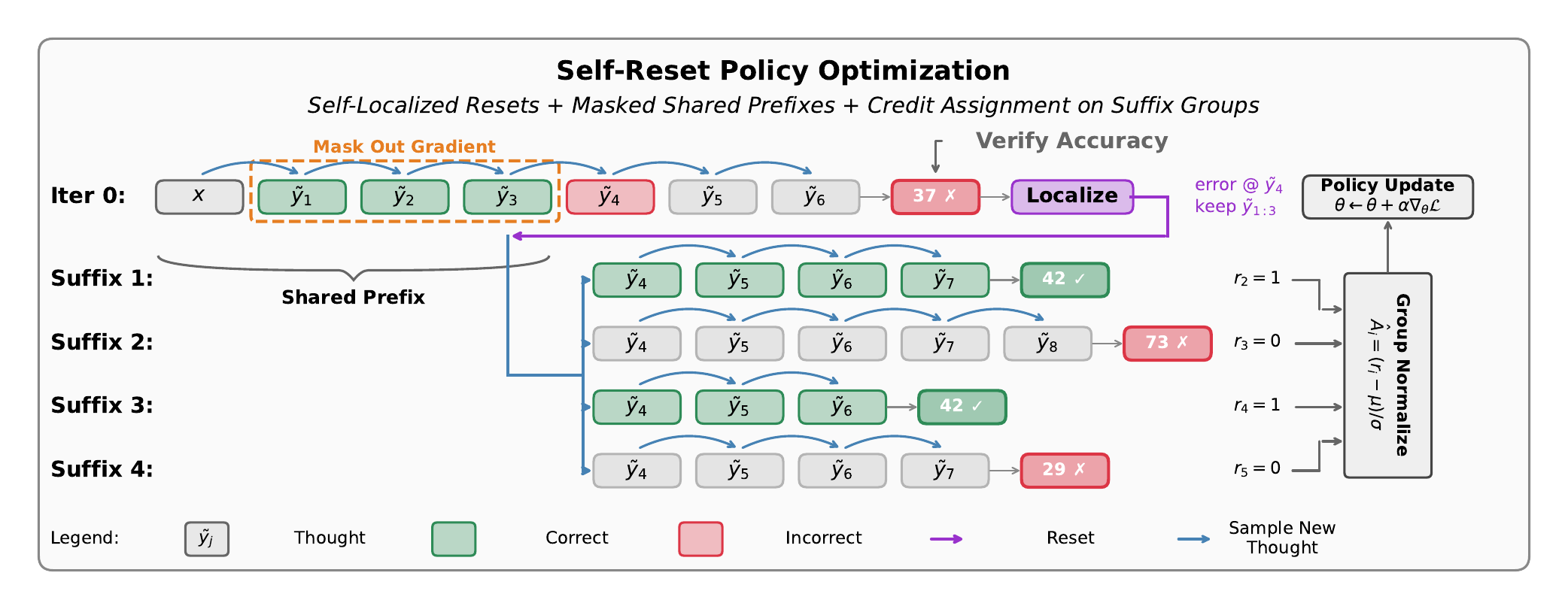}
\caption{Self-Reset Policy Optimization (\SRPO{}). The base policy
  produces an initial rollout (Iter~0) that fails. The model self-localizes
  the first erroneous thought ($\tilde{y}_4$) and resets to the improvable
  prefix $\tilde{y}_{1:3}$. From this shared prefix, multiple suffix
  rollouts are sampled, each receiving a reward $r_i$. Group normalization
  converts rewards to advantages $\hat{A}_i$, which drive the policy
  update; the gradient on the shared prefix is masked out so the update is
  concentrated on the suffix tokens.}
\label{fig:srpo_overview}
\end{figure}

\section{Preliminaries}
\label{sec:preliminaries}

We study finite-horizon Markov Decision Processes (MDPs) specified by a
tuple $(\mathcal{X}, \mathcal{Y}, P, r, H, \mu)$, where $\mathcal{X}$ is
the state space, $\mathcal{Y}$ is the action space, $H \geq 1$ is the
horizon, $\mu \in \Delta(\mathcal{X})$ is the initial state distribution,
$P = (P_1, \ldots, P_H)$ is a sequence of transition kernels
$P_h : \mathcal{X} \times \mathcal{Y} \to \Delta(\mathcal{X})$, and
$r = (r_1, \ldots, r_H)$ is a sequence of reward functions
$r_h : \mathcal{X} \times \mathcal{Y} \to \mathbb{R}$.

A policy $\pi = (\pi_1, \ldots, \pi_H)$ is a sequence of decision rules
$\pi_h : \mathcal{X} \to \Delta(\mathcal{Y})$. Executing $\pi$ from
$x_1 \sim \mu$ induces a trajectory $(x_1, y_1, \ldots, x_H, y_H)$ via
$y_h \sim \pi_h(\cdot \mid x_h)$ and
$x_{h+1} \sim P_h(\cdot \mid x_h, y_h)$, and yields return
$\sum_{h=1}^{H} r_h(x_h, y_h)$. The objective is the expected return $J(\pi) \coloneqq \mathbb{E}_\pi\!\left[\sum_{h=1}^{H} r_h(x_h, y_h)\right]$.

For a policy $\pi$ and time $h$, the value function, action-value
function, and advantage function are defined for all
$x \in \mathcal{X}$ and $y \in \mathcal{Y}$ as
\[
  V_h^\pi(x)
  \;\coloneqq\; \mathbb{E}_\pi\!\left[\textstyle\sum_{t=h}^{H} r_t(x_t, y_t) \,\Big|\, x_h = x\right],
  \qquad
  Q_h^\pi(x, y)
  \;\coloneqq\; r_h(x, y) + \mathbb{E}_{x' \sim P_h(\cdot \mid x, y)}\!\left[V_{h+1}^\pi(x')\right],
\]
\[
  A_h^\pi(x, y) \;\coloneqq\; Q_h^\pi(x, y) - V_h^\pi(x),
\]
with the convention $V_{H+1}^\pi \equiv 0$. The time-$h$ state
distribution under $\pi$ is
$d_\pi^h(x) \coloneqq \mathbb{P}_\pi(x_h = x \mid x_1 \sim \mu)$ for all
$x \in \mathcal{X}$, and the time-averaged state distribution is
$d_{\pi,\mu}(x) \coloneqq \frac{1}{H}\sum_{h=1}^{H} d_\pi^h(x)$ for all
$x \in \mathcal{X}$.

The \emph{greedy policy} with respect to $\pi$ is the sequence
$\pi^+ = (\pi^+_1, \ldots, \pi^+_H)$ with
$\pi^+_h(x) \coloneqq \argmax_{y \in \mathcal{Y}} A_h^\pi(x, y)$ for all
$x \in \mathcal{X}$. For any policy $\pi'$, the \emph{policy advantage}
of $\pi'$ against $\pi$ is
$
  \mathbb{A}_{\pi,\mu}(\pi')
  \;\coloneqq\; \frac{1}{H} \sum_{h=1}^{H} \mathbb{E}_{x \sim d_\pi^h}\,\mathbb{E}_{y \sim \pi'_h(\cdot \mid x)}\!\left[A_h^\pi(x, y)\right].
$

\section{From Random to Credit-Assignment Based Resets in CPI}
\label{sec:theory}

Conservative Policy Iteration (CPI; \citealp{kakade2002cpi}) is a
foundational algorithm for approximate policy improvement, with
provably monotone-improvement guarantees. CPI's framework underlies
TRPO \citep{schulman2015trpo}, PPO \citep{schulman2017ppo}, MDPO~\citep{tomar2020mirror}, and GRPO
\citep{shao2024deepseekmath}, which are commonly used in modern language
model post-training. Each CPI update can be
interpreted as a small gradient step on the policy, a Frank--Wolfe
\citep{frank1956} update on $J(\pi)$
\citep{vieillard2020deep,sherman2025convergence}. We refer to this procedure as
\CPIRR{}.

CPI improves a policy $\pi$ by \emph{random-reset sampling}: draw
$x \sim d_\pi^h$ at a uniformly-random time $h$, take a random action
$y \sim \mathrm{Unif}(\mathcal{Y})$, and roll out $\pi$ to the horizon
to obtain an unbiased estimate of $Q_h^\pi(x, y)$. Aggregating these
estimates fits $A_h^\pi$ and yields the greedy update direction $\pi^+$.
$\pi$ is then updated via the conservative mixture
$\pi_\alpha = (1-\alpha)\pi + \alpha\pi^+$; for a suitable step size
$\alpha$, CPI guarantees monotone improvement at a rate governed by the
policy advantage $\mathbb{A}_{\pi,\mu}(\pi^+)$.

A closer look reveals that the improvement magnitude is governed by
states with a specific property. $\mathbb{A}_{\pi,\mu}(\pi^+)$, an
expectation of $\pi^+$'s advantage over $\pi$ under on-policy
visitation, is dominated by \emph{large-advantage states} where $\pi$
admits a strictly better action. Random-reset sampling draws uniformly
from $d_{\pi,\mu}$, diluting this signal rather than concentrating it
where the gain lives. Modern language models can target such states
directly. Given a failed reasoning trajectory, they self-localize
the erroneous thought, precisely a state
where a different action would have changed the outcome \citep{yang2026int, samanta2026structure}. We abstract
this capability as a \emph{credit-assignment oracle} and study the
resulting CPI variant, \CPICARO{}, in the next subsection.
Table~\ref{tab:comparison} previews its provably better sample
complexity and per-iteration improvement over \CPIRR{}.

\begin{table}[t]
\centering
\renewcommand{\arraystretch}{1.4}
\caption{Sample complexity and per-iteration improvement for \CPIRR{} and \CPICARO{} in the regime $\mathbb{A}^{\mathcal{G}^c}_{\pi,\mu}(\pi^+) \ll \tau$. \CPICARO{} saves a $p_\pi^{\,2}$ factor in sample complexity and gains a $1/p_\pi$ factor in per-iteration improvement. In the RLVR setting studied in Section~\ref{sec:methods_main}, we define a state as improvable if an error was made during its reasoning trace. Under this definition, $\mathbb{A}^{\mathcal{G}^c}_{\pi,\mu}(\pi^+) = 0$: if no error was made, the agent outputs the correct solution and there is no opportunity to improve upon it. See Theorem~\ref{thm:main} and 
Section~\ref{subsec:gap}.}
\label{tab:comparison}
\begin{tabular}{lcc}
\toprule
 & Samples & Per-iteration improvement \\
\midrule
\textcolor{maroon}{\CPIRR{}} &
  $\tilde O\!\left(\dfrac{|\mathcal{Y}|\, H^2\, R_{\max}^2}{\tau^2\, p_\pi^{\,2}}\right)$ &
  $\Omega\!\left(\dfrac{\tau^2\, p_\pi^{\,2}}{H\, R_{\max}}\right)$ \\
\textcolor{blue}{\CPICARO{}} &
  $\tilde O\!\left(\dfrac{|\mathcal{Y}|\, H^2\, R_{\max}^2}{\tau^2}\right)$ &
  $\Omega\!\left(\dfrac{\tau^2\, p_\pi}{H\, R_{\max}}\right)$ \\
\bottomrule
\end{tabular}
\end{table}

\subsection{CPI with a Credit-Assignment Reset Oracle}
\label{subsec:cpi-oracle}

We now formalize the credit-assignment oracle and analyze the CPI
variant that uses it.

\paragraph{Formalizing the Credit-Assignment Reset Oracle.}
Fix a threshold $\tau > 0$. For each $h \in [H]$, the
\emph{$\tau$-improvable set at time $h$} is $\mathcal{G}_{h,\tau} \coloneqq \bigl\{ x \in \mathcal{X} : \max_{y \in \mathcal{Y}} A_h^\pi(x, y) \geq \tau \bigr\}$.
Let $p_{\pi, h} \coloneqq \mathbb{P}_{x \sim d_\pi^h}[x \in \mathcal{G}_{h,\tau}]$
denote the time-$h$ coverage of $\mathcal{G}_{h,\tau}$, and let the
\emph{coverage} of $\pi$ at threshold $\tau$ be the time-averaged
quantity $p_\pi \coloneqq \frac{1}{H}\sum_{h=1}^{H} p_{\pi, h}$. We
assume $p_\pi > 0$ throughout this section. If $p_\pi = 0$ then
$\pi_\mathcal{G} \equiv \pi$ and the algorithm makes no update.
The \emph{credit-assignment greedy policy} plays $\pi^+$ on the
improvable set and $\pi$ elsewhere: for all $h \in [H]$,
$x \in \mathcal{X}$, $y \in \mathcal{Y}$,
\[
  (\pi_\mathcal{G})_h(y \mid x)
  \;\coloneqq\;
  \begin{cases}
    \pi^+_h(y \mid x) & \text{if } x \in \mathcal{G}_{h,\tau}, \\
    \pi_h(y \mid x) & \text{otherwise.}
  \end{cases}
\]
A \emph{credit-assignment oracle} is a membership test
$\mathcal{O}(x, h) = \mathbf{1}\{x \in \mathcal{G}_{h,\tau}\}$, and the
corresponding \emph{credit sampler} $\textsc{CreditSampler}(\pi, \mu)$
returns pairs $(x, h)$ with $h \sim \mathrm{Unif}[H]$,
$x \sim d_\pi^h$, conditioned on $x \in \mathcal{G}_{h,\tau}$. Finally, for
any policy $\pi'$, the \emph{conditional policy advantages} on and off
$\mathcal{G}_{h,\tau}$ are
$\mathbb{A}^{\mathcal{G}}_{\pi,\mu}(\pi') \coloneqq \mathbb{E}[A_h^\pi(x, y) \mid x \in \mathcal{G}_{h,\tau}]$
and
$\mathbb{A}^{\mathcal{G}^c}_{\pi,\mu}(\pi') \coloneqq \mathbb{E}[A_h^\pi(x, y) \mid x \notin \mathcal{G}_{h,\tau}]$,
with expectations taken under $h \sim \mathrm{Unif}[H]$,
$x \sim d_\pi^h$, $y \sim \pi'_h(\cdot \mid x)$.

Since $\pi_\mathcal{G}$ deviates from $\pi$ only on $\mathcal{G}_{h,\tau}$,
its policy advantage simplifies to
$\mathbb{A}_{\pi,\mu}(\pi_\mathcal{G}) = p_\pi\, \mathbb{A}^{\mathcal{G}}_{\pi,\mu}(\pi^+) \geq \tau\, p_\pi$.
Knowing the membership in $\mathcal{G}_{h,\tau}$, the credit sampler
estimates this conditional advantage $\mathbb{A}^{\mathcal{G}}_{\pi,\mu}(\pi^+) \geq \tau$
directly, a $1/p_\pi$-larger signal than the diluted
$\mathbb{A}_{\pi,\mu}(\pi^+) \geq \tau p_\pi$ that random-reset
sampling targets. This motivates a CPI variant designed around the
oracle, which we describe next.

\paragraph{Algorithms.}
We adapt \CPIRR{} into \CPICARO{}, a variant that uses the credit
sampler to target states with large potential improvement.
Algorithm~\ref{alg:cpi-oracle} presents both in a single block.
\CPICARO{} differs from \CPIRR{} in exactly two steps: it samples
$(x, h)$ from the credit sampler rather than the on-policy
distribution, and applies the conservative update only on
$\mathcal{G}_{h,\tau}$.

\begin{algorithm}[htbp]
\caption{\textcolor{maroon}{\CPIRR{}} /
  \textcolor{blue}{\CPICARO{}}}
\label{alg:cpi-oracle}
\begin{algorithmic}[1]
\Require policy $\pi$, function class $\mathcal{F}$, threshold $\tau$, sample size $n$
\Statex \textit{// Phase 1: Q-sampling}
\For{$i = 1, \ldots, n$}
  \State \textcolor{maroon}{$(x_i, h_i) \gets \textsc{Sampler}(\pi, \mu)$} \Comment{\textcolor{maroon}{on-policy draw}}
  \State \textcolor{blue}{$(x_i, h_i) \gets \textsc{CreditSampler}(\pi, \mu)$} \Comment{\textcolor{blue}{draw conditioned on $\mathcal{G}_{h,\tau}$}}
  \State sample $y_i \sim \mathrm{Unif}(\mathcal{Y})$
  \State $\hat Q_i \gets \textsc{QRollout}(\pi, x_i, y_i, h_i)$
\EndFor
\Statex \textit{// Phase 2: fit, derive greedy, estimate advantage}
\State $\hat Q \gets \argmin_{f \in \mathcal{F}} \sum_{i=1}^n (f(x_i, y_i, h_i) - \hat Q_i)^2$
\State $\hat\pi^+_h(x) \gets \argmax_{y \in \mathcal{Y}} \hat Q(x, y, h)$
\State $\hat{\mathbb{A}} \gets \tfrac{1}{n}\sum_{i=1}^n |\mathcal{Y}| \bigl(\hat\pi^+_{h_i}(y_i \mid x_i) - \pi_{h_i}(y_i \mid x_i)\bigr)\hat Q(x_i, y_i, h_i)$
\Statex \textit{// Phase 3: conservative update}
\State \textcolor{maroon}{$\hat\alpha \gets \min\{1,\; \hat{\mathbb{A}} / (H^2 R_{\max})\}$} \Comment{\textcolor{maroon}{step from \citet{kakade2002cpi}}}
\State \textcolor{blue}{$\hat\alpha \gets \min\{1,\; \hat{\mathbb{A}} / (2\, H^2 R_{\max})\}$} \Comment{\textcolor{blue}{step from Cor.~\ref{cor:credit-cpi} (Appendix~\ref{app:lemmas})}}
\State \textcolor{maroon}{\Return $\pi_{\hat\alpha} \gets (1-\hat\alpha)\,\pi + \hat\alpha\,\hat\pi^+$ on all states}
\State \textcolor{blue}{\Return $\pi_{\hat\alpha} \gets (1-\hat\alpha)\,\pi + \hat\alpha\,\hat\pi^+$ on $\mathcal{G}_{h,\tau}$;\; $\pi$ on $\mathcal{G}_{h,\tau}^c$}
\end{algorithmic}
\end{algorithm}

\paragraph{Implementing the sampler.}
Given the credit-assignment oracle $\mathcal{O}$, the credit sampler is
implemented via rejection sampling: draw $(x, h)$ from the on-policy
distribution, accept if $\mathcal{O}(x, h) = 1$, reject otherwise. This
costs $O(n / p_\pi)$ cheap queries for $n$ accepted samples, with
expensive Q-rollouts run only on accepted draws. We defer the formal
procedure to Appendix~\ref{app:rejection-sampling}.

\paragraph{Main guarantee.}
Theorem~\ref{thm:main} quantifies the sample complexity and
per-iteration improvement of both variants (see
Appendix~\ref{app:proof-main} for the proof).

\begin{theorem}[Sample complexity and per-iteration improvement]
\label{thm:main}
Assume the function class $\mathcal{F}$ is finite and realizable, namely, $Q_h^\pi \in \mathcal{F}$ and
rewards are bounded in $r_h \in[0, R_{\max}]$ for all $h \in [H]$. Fix $\delta \in (0, 1)$ and define $N_0 := C\, |\mathcal{Y}|\, H^2 R_{\max}^2 \log(|\mathcal{F}|/\delta)$ for some constant $C > 0$.
With probability at least $1 - \delta$, Algorithm~\ref{alg:cpi-oracle}
returns a policy $\pi_{\hat\alpha}$ satisfying:
\begin{enumerate}[label=\alph*), leftmargin=*]
\item \textcolor{maroon}{\CPIRR{}}. With sample size
$n \,\geq\, \frac{N_0}{\tau^2\, p_\pi^{\,2}}$, it holds that
$$
  J(\pi_{\hat\alpha}) - J(\pi) \;\geq\; \frac{\mathbb{A}_{\pi,\mu}(\pi^+)^2}{8\, H\, R_{\max}} \;\geq\; \frac{\tau^2\, p_\pi^{\,2}}{8\, H\, R_{\max}},
$$
where the second inequality is tight in the regime $\mathbb{A}^{\mathcal{G}^c}_{\pi,\mu}(\pi^+) \ll \tau$.
  \item \textcolor{blue}{\CPICARO{}}. With sample size
  $n \,\geq\, \frac{N_0}{\tau^2}$, it holds that
  $$
    J(\pi_{\hat\alpha}) - J(\pi) \;\geq\; \frac{\tau^2\, p_\pi}{16\, H\, R_{\max}}.
  $$
\end{enumerate}
\end{theorem}

The gain is most pronounced when $p_\pi$ is small, exactly the regime where the probability to hit a state in the improvable set is small and \CPIRR{} spends most of its budget on states without signal. The improvement of \CPICARO{} over \CPIRR{} is largest when $\mathbb{A}^{\mathcal{G}^c}_{\pi,\mu}(\pi^+) \ll \tau$, in which case $\mathbb{A}_{\pi,\mu}(\pi^+) \sim p_\pi\, \mathbb{A}^{\mathcal{G}}_{\pi,\mu}(\pi^+) \geq p_\pi \tau$ (see Lemma~\ref{lem:advantage-decomp}). Next, we explain the mechanism behind both improvements.

\subsection{Where Does the Improvement Come From?}
\label{subsec:gap}

CPI is built around the lower bound \citep{kakade2002cpi}:
\begin{equation}
J(\pi_\alpha) - J(\pi) \;\geq\; \alpha\, H\, \textcolor{maroon}{\mathbb{A}_{\pi,\mu}(\pi^+)} - \tfrac{1}{2}\alpha^2 H^3 R_{\max}. \label{eq:cpi-classical}
\end{equation}
Maximizing the right-hand side on $\alpha$ recovers the step
$\alpha^* \propto \mathbb{A}_{\pi,\mu}(\pi^+)/(H^2 R_{\max}) \sim \tau p_\pi /(H^2 R_{\max})$ assuming that the advantage of $\pi^+$ on states not in $\mathcal{G}$ is vanishing, namely $\mathbb{A}^{\mathcal{G}^c}_{\pi,\mu}(\pi') \ll \tau$. Both scale with
the lower bound on the policy advantage that the algorithm can
guarantee.
Without oracle access, the only lower bound on $\mathbb{A}_{\pi,\mu}(\pi^+)$ derivable from $\tau$ and $p_\pi$ is $\tau p_\pi$. Estimating this minimum guaranteed advantage requires $n \gtrsim |\mathcal{Y}|\, H^2 R_{\max}^2 / (\tau^2 p_\pi^{\,2})$ samples , matching \CPIRR{}'s rate in Theorem~\ref{thm:main} and shown tight via a finite-sample Cramer bound (see Appendix~\ref{app:tightness}).

For \CPICARO{}, the update applies $\pi^+$ only on $\mathcal{G}_{h,\tau}$.
This restriction yields a credit-aware simulation lemma and the
following credit-aware CPI lower bound (see
Appendix~\ref{app:lemmas}),
\begin{equation}
J(\pi_\alpha) - J(\pi) \;\geq\; \alpha\, H\, \textcolor{blue}{p_\pi}\, \textcolor{blue}{\mathbb{A}^{\mathcal{G}}_{\pi,\mu}(\pi^+)} - \alpha^2 H^3 R_{\max}\, \textcolor{blue}{p_\pi}. \label{eq:cpi-credit}
\end{equation}
Maximizing the right-hand side now requires estimating
$\mathbb{A}^{\mathcal{G}}_{\pi,\mu}(\pi^+) \geq \tau$, a $1/p_\pi$-larger target.
The improvement therefore comes from two effects, acting on different parts of \CPICARO{}.
The credit-aware simulation lemma allows a $1/p_\pi$-larger step size
$\hat\alpha \propto  \mathbb{A}^{\mathcal{G}}_{\pi,\mu}(\pi^+) /(H^2 R_{\max})\sim \tau /(H^2 R_{\max}) $, yielding $1/p_\pi$ more improvement per iteration. Further, the
larger estimation target reduces the sample budget $n$ at a given
precision by $1/p_\pi^2$. Together they yield the rate for \CPICARO{} as shown in Theorem~\ref{thm:main}.

\section{Methods}
\label{sec:methods_main}

We translate the theory of Section~\ref{sec:theory} into a practical
post-training RLVR method through three choices: the granularity of
credit assignment (Section~\ref{sec:thought-mdp}), how to reach and
resample from improvable states (Section~\ref{sec:srpo}), and how to
reinforce the rollouts (Section~\ref{sec:loss}).

\subsection{Granularity of Credit Assignment}
\label{sec:thought-mdp}

Our methods use a \emph{thought-level} abstraction, where each action is a semantically coherent reasoning step the model emits as it generates its completion. Other granularities used in prior work include human- or judge-annotated step boundaries~\citep{lightman_lets_2023, uesato_solving_2022, wang_math-shepherd_2023, luo_improve_2024}, fixed-length token blocks~\citep{wang2026value}, and heuristic entropy- or token-count-based segments~\citep{guo_segment_2025}. Token-level actions are too fine for sparse-reward credit assignment, while these higher-level abstractions impose boundaries the model itself did not choose, risking incomplete segments. For credit assignment with resets, self-determining each thought's boundaries during generation makes every thought an atomic, self-contained unit, requiring no retroactive parsing. Importantly, \citet{samanta2026structure} shows that language models can reliably localize erroneous thought steps within such structured reasoning traces, making self-localized resets feasible.

We formalize the thought-level generation process as a \emph{Thought MDP}. The initial state $x_0 \sim \mu$ is
the prompt. At step $h$, the model emits a thought
$\tilde y_h \sim \pi_h(\cdot \mid x_0, \tilde y_{1:h-1})$ ---
a sequence of tokens delimited by a stop pattern --- and the prefix
$(x_0, \tilde y_{1:h-1})$ fully describes the state. Reward is
sparse: $r_h = 0$ for $h < H$, with terminal $r_H \in \{0, 1\}$
from oracle verification of final-answer correctness.

This granularity has two practical payoffs. First, since the
thought structure is induced by prompting alone (rollout prompt in
Appendix~\ref{app:prompts}), resetting to any prefix state and
resampling the next thought when generating rollouts introduces no distributional shift ---
a precondition for the on-policy reset-and-resample procedure of
Section~\ref{sec:srpo}. Second, the improvable set $\mathcal{G}_{h,\tau}$
consists of recoverable reasoning errors rather than arbitrary
token positions, giving self-localization a meaningful target.


\subsection{Generating Rollouts by Resetting and Resampling}
\label{sec:srpo}
This section describes how \RRPO{} and \SRPO{} construct training rollouts. Both realize their theoretical counterparts at the thought level: \RRPO{} corresponds to \CPIRR{}, and \SRPO{} corresponds to \CPICARO{} in the RLVR setting, with self-localization of the first erroneous step standing in for the credit-assignment oracle.

At each training step, we build a buffer that combines a base group of $G$ iid rollouts from the prompt $x_0$ under standard GRPO with a shared-prefix group of $G$ rollouts from a reset state $x^*$ (see \Algref{alg:slrpo}). To construct the shared-prefix group, we draw iid rollouts from $\pi_\theta(\cdot \mid x_0)$ until the first incorrect one, the \emph{seed} of length $H_s$ thoughts, then pick a reset index $h^*$, form $x^* = (x_0, \tilde y_{1:h^*-1})$, and sample $G$ iid thought-level rollouts from $x^*$. The seed is excluded from the base group to preserve iid sampling. The two methods differ only in how $h^*$ is chosen. Figure~\ref{fig:sampling_strategies} illustrates the resulting rollout buffer alongside the GRPO baseline and two alternative splits, $2{\times}4$ and $1{\times}8$, studied in Section~\ref{sec:sampling_design}. If sampling exhausts without an incorrect rollout, which is rare in practice, the step falls back to standard GRPO on the base group.

\paragraph{\RRPO{}: random reset.}
\RRPO{} draws $h^* \sim \mathrm{Unif}\{1,\ldots,H_s\}$, so the reset state $x^*$ is an arbitrary thought-level prefix of the failed seed. This realizes \CPIRR{}'s uniform reset at the thought level.

\paragraph{\SRPO{}: self-localized reset.}
\SRPO{} sets $h^*$ via error localization. The localization signal could come from a PRM, step-level environment feedback, or a separate localizer. To avoid auxiliary supervision, we instantiate it with explicit self-localization, where the same policy $\pi_\theta$ that generated the seed is prompted to analyze its own reasoning trace. The seed is presented as the labeled sequence of thoughts $\tilde y_1, \ldots, \tilde y_{H_s}$, and the model returns the index of the first incorrect thought (see localization prompt in Appendix~\ref{app:prompts}). The reset state $x^*$ is then the self-determined verified-correct prefix preceding that error. This realizes \CPICARO{} at the thought level, with the localizer taking the place of the rejection-sampling draw from $\mathcal{G}_{h,\tau}$ in Algorithm~\ref{alg:rejection}.

\begin{algorithm}[htbp]
\caption{\textcolor{maroon}{\RRPO{}} / \textcolor{blue}{\SRPO{}}}
\label{alg:slrpo}
\begin{algorithmic}[1]
\Require prompt $x_0$, policy $\pi_\theta$, group size $G$
\Statex \textit{// Phase 1: base group (on-policy)}
\State sample iid rollouts from $x_0$ until the first incorrect one; let it be the \emph{seed}, with thoughts $\tilde y_1, \ldots, \tilde y_{H_s}$
\State collect all preceding correct rollouts as initial members of group 1; discard the seed
\State sample additional iid rollouts from $x_0$ until $|\text{group 1}| = G$
\Statex \textit{// Phase 2: shared-prefix group at reset state $x^*$}
\State \textcolor{maroon}{$h^* \sim \mathrm{Unif}\{1,\ldots,H_s\}$} \Comment{\textcolor{maroon}{random reset}}
\State \textcolor{blue}{$h^* \gets$ self-localized index of the seed's first erroneous thought} \Comment{\textcolor{blue}{oracle reset}}
\State $x^* \gets (x_0, \tilde y_{1:h^*-1})$
\State sample $G$ iid rollouts from $x^*$ $\to$ group 2
\Statex \textit{// Phase 3: advantage estimation (group-normalized within each group)}
\For{each group $k \in \{1, 2\}$}
  \State $\hat{A}_i \gets (r_i - \operatorname{mean}(\mathbf{r}_k)) / \operatorname{std}(\mathbf{r}_k)$ for each rollout $i$ in group $k$
\EndFor
\Statex \textit{// Phase 4: update}
\State optimize GRPO loss over all $2G$ rollouts with advantages $\hat{A}$; mask the prefix tokens of $x^*$ in group 2
\end{algorithmic}
\end{algorithm}

\subsection{Reinforcing Rollouts via Group-Relative Advantages}
\label{sec:loss}

With the two-group rollout buffer assembled, we now turn to the
policy update. We learn from each rollout group via group-relative
normalization \citep{shao2024deepseekmath}. For a group of $G$ rollouts with
outcome rewards $r_1, \ldots, r_G$, rollout $i$ receives the
self-normalized advantage $\hat{A}_i = (r_i - \mu)/\sigma$, where
$\mu$ and $\sigma$ are the group's empirical mean and standard
deviation; each $\hat{A}_i$ measures rollout $i$'s outcome relative
to the rest of its group. With rollouts sampled from the prompt $x_0$ under standard GRPO, this advantage is broadcast uniformly across every token in the rollout.

For the shared-prefix group, each rollout's suffix is a token sequence $y_{i,1}, \ldots, y_{i,T_i}$ of length $T_i$ sampled from $\pi_\theta(\cdot \mid x^*)$ to termination, where $x^* = (x_0, \tilde y_{1:h^*-1})$ is the reset state. We mask the shared prefix and apply the group-relative advantage only to suffix tokens, giving the shared-prefix loss
$\mathcal{L}_{\mathrm{SP}}(\theta) = -\frac{1}{G} \sum_{i=1}^{G} \frac{1}{T_i} \sum_{t=1}^{T_i} \hat{A}_i \log \pi_\theta(y_{i,t} \mid x^*, y_{i,<t})$.
This prefix masking is the parametric analog of \CPICARO{}'s policy update on the reset state. Because $x^*$ is itself a thought-level prefix, the resulting credit signal lives on thoughts rather than tokens. The base group uses the analogous loss over full rollouts. We take one on-policy gradient step per group, without PPO clipping or KL regularization (Appendix~\ref{app:training_details}, Table~\ref{tab:ablation_clip} shows that adding PPO clipping does not help).

\section{Results}
\label{sec:results}

We train Qwen2.5-14B-Instruct~\citep{Yang2024Qwen25TR} and OLMo-3-7B-Instruct~\citep{olmo2025olmo3} on NuminaMath-Olympiads~\citep{numina_math_datasets} using 400 problems and 2 epochs, and thoroughly test general reasoning performance over a suite of 10 benchmark tasks covering math, science, and strategic reasoning: NuminaMath-Olympiads ~\citep{numina_math_datasets}, HMMT Nov 2025 ~\citep{hmmt2025nov}, MATH Level-5 ~\citep{hendrycksmath2021}, StrategyQA ~\citep{tacl2021strategyqa}, AceReason-Math ~\citep{chen2025acereason1.1}, SciKnowEval Chemistry ~\citep{feng2024sciknoweval}, SciKnowEval Biology ~\citep{feng2024sciknoweval}, CommonsenseQA ~\citep{talmor-etal-2019-commonsenseqa}, SciKnowEval Materials ~\citep{feng2024sciknoweval}, and SciKnowEval Physics ~\citep{feng2024sciknoweval}. For each benchmark we evaluate on 500 randomly selected test items. We aggregate all results over 3 seeds; tables report mean $\pm$ SD with the best per column in bold. Numbers represent final performance after 2 epochs of training. Evaluation uses temperature $0$. Each method is roughly compute-matched to result in 8 rollouts per prompt. We train all methods with LoRA adapters \citep{hu2021lora} of rank $64$ and $\alpha = 64$. Further training details are in Appendix~\ref{app:training_details}.

\subsection{Evaluating Reset-Based Sampling Strategies}
\label{sec:sampling_design}

We first study how to effectively utilize resets under a fixed compute budget. Holding total compute fixed at 8 rollouts per prompt, we compare three \SRPO{} instantiations (Figure~\ref{fig:sampling_strategies}): \emph{1$\times$4} splits the budget as 4 iid base rollouts plus 4 suffixes resampled from one self-localized prefix; \emph{2$\times$4} doubles prefix diversity (two prefixes, 4 suffixes each) at the cost of the base group; \emph{1$\times$8} maximizes suffix depth from a single localization.

\begin{figure}[b]
\centering
\includegraphics[width=\textwidth]{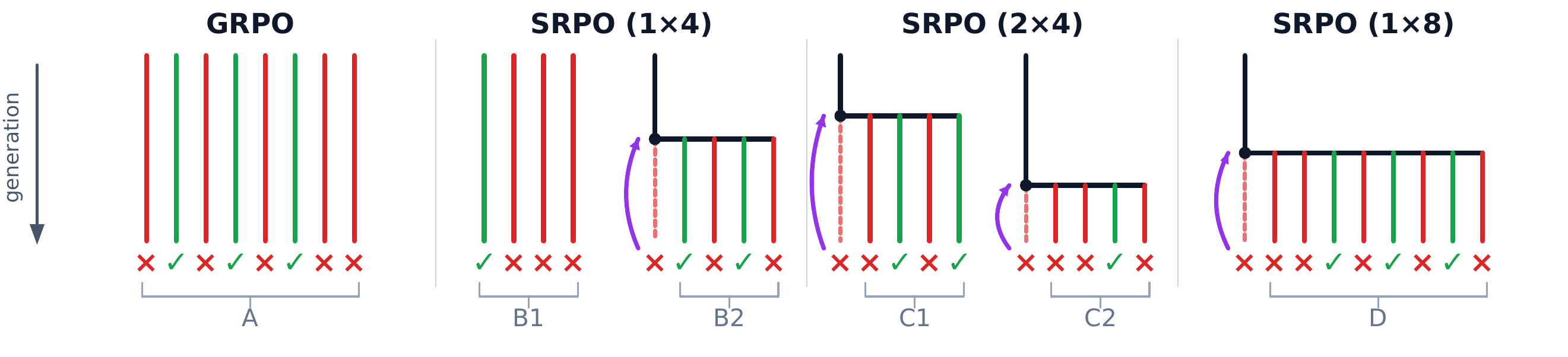}
\caption{Reset-based sampling strategies under a fixed 8-rollout-per-prompt budget.}
\label{fig:sampling_strategies}
\end{figure}

\begin{table*}[htbp]
\centering
\scriptsize
\setlength{\tabcolsep}{2pt}
\caption{Sampling-strategy comparison for \SRPO{} under a fixed 8-rollout-per-prompt budget on OLMo-3-7B-Instruct. Full results including Qwen2.5-14B-Instruct in Appendix~\ref{app:sampling_full}.}
\label{tab:ablation_sampling}
\begin{tabular}{l*{10}{c}}
\toprule
Method & oly & hmmt & lvl5 & stra & ace & chem & bio & csqa & mat & phys \\
\midrule
\SRPO{} (1$\times$4) & 24.8\,$\pm$\,1.9 & \textbf{15.6\,$\pm$\,4.2} & \textbf{67.5\,$\pm$\,3.5} & \textbf{66.6\,$\pm$\,2.5} & \textbf{59.6\,$\pm$\,2.6} & \textbf{24.7\,$\pm$\,1.4} & 28.7\,$\pm$\,0.6 & \textbf{74.8\,$\pm$\,1.0} & \textbf{55.9\,$\pm$\,0.6} & \textbf{54.5\,$\pm$\,2.0} \\
\SRPO{} (2$\times$4) & \textbf{26.1\,$\pm$\,1.2} & 14.4\,$\pm$\,3.1 & 66.0\,$\pm$\,3.1 & 65.9\,$\pm$\,2.2 & 56.7\,$\pm$\,2.7 & 22.4\,$\pm$\,0.5 & 27.9\,$\pm$\,0.2 & 65.9\,$\pm$\,3.3 & 45.7\,$\pm$\,1.3 & 43.2\,$\pm$\,5.7 \\
\SRPO{} (1$\times$8) & 23.1\,$\pm$\,2.6 & 12.2\,$\pm$\,3.1 & 62.3\,$\pm$\,2.5 & 64.1\,$\pm$\,3.4 & 53.7\,$\pm$\,6.3 & 22.8\,$\pm$\,1.2 & \textbf{29.8\,$\pm$\,1.9} & 59.5\,$\pm$\,14.9 & 43.5\,$\pm$\,5.8 & 43.4\,$\pm$\,8.6 \\
\bottomrule
\end{tabular}
\end{table*}

The 1$\times$4 split wins on the majority of columns, balancing base-policy coverage with shared-prefix depth; the same pattern holds for Qwen2.5-14B-Instruct (Appendix~\ref{app:sampling_full}). We adopt 1$\times$4 as the default \SRPO{} configuration for the remaining experiments.

\subsection{General Reasoning Performance}
\label{sec:main_results}

We next compare \SRPO{} and \RRPO{} against GRPO and related baselines (App.~\ref{app:related_work}) in Tab.~\ref{tab:aggregate}. \SRPO{} is the strongest non-baseline method on both base models, with the best result on 7/10 tasks for Qwen2.5-14B and 6/10 for OLMo-3-7B. Since training uses NuminaMath-Olympiads alone, the gains on the other 9 benchmarks (science, strategic, commonsense, and other math tasks) reflect out-of-distribution generalization. \RRPO{} is comparable to GRPO.

\begin{table*}[htbp]
\centering
\scriptsize
\setlength{\tabcolsep}{2pt}
\caption{General
reasoning performance across math, science, strategic, and commonsense benchmarks.}
\label{tab:aggregate}
\begin{tabular}{l*{10}{c}}
\toprule
Method & oly & hmmt & lvl5 & stra & ace & chem & bio & csqa & mat & phys \\
\midrule
\multicolumn{11}{l}{\textit{Qwen2.5-14B-Instruct}} \\
\midrule
Base$^\dagger$ & 26.0 & 6.7 & 51.0 & 71.4 & 44.2 & 23.6 & 30.0 & 65.8 & 30.0 & 26.6 \\
GRPO & 25.0\,$\pm$\,1.5 & 5.6\,$\pm$\,1.6 & 52.1\,$\pm$\,1.2 & 69.5\,$\pm$\,2.8 & 45.2\,$\pm$\,2.0 & \textbf{26.5\,$\pm$\,2.2} & 30.1\,$\pm$\,1.1 & 66.1\,$\pm$\,16.7 & 29.4\,$\pm$\,5.7 & 26.9\,$\pm$\,6.8 \\
SCoRe & 25.3\,$\pm$\,1.1 & 5.6\,$\pm$\,1.6 & 54.1\,$\pm$\,0.5 & 72.0\,$\pm$\,0.5 & 43.7\,$\pm$\,2.1 & 25.0\,$\pm$\,3.2 & 30.2\,$\pm$\,1.9 & 78.0\,$\pm$\,1.7 & 30.9\,$\pm$\,2.3 & 29.5\,$\pm$\,4.1 \\
Cr-GRPO & 24.9\,$\pm$\,0.7 & \textbf{7.8\,$\pm$\,1.6} & 51.1\,$\pm$\,0.2 & 72.1\,$\pm$\,0.4 & 42.5\,$\pm$\,1.2 & 22.9\,$\pm$\,0.5 & 30.9\,$\pm$\,1.2 & 66.8\,$\pm$\,0.7 & 29.8\,$\pm$\,1.4 & 24.9\,$\pm$\,1.2 \\
SPO-Tree & 25.1\,$\pm$\,1.1 & 4.4\,$\pm$\,4.2 & 51.5\,$\pm$\,1.2 & 71.7\,$\pm$\,1.2 & 43.1\,$\pm$\,0.4 & 23.4\,$\pm$\,0.9 & 31.4\,$\pm$\,0.9 & 66.7\,$\pm$\,0.9 & 29.0\,$\pm$\,0.2 & 24.5\,$\pm$\,1.1 \\
\cmidrule(lr){1-11}
\RRPO{} & \textbf{25.7\,$\pm$\,0.9} & \textbf{7.8\,$\pm$\,1.6} & 52.7\,$\pm$\,1.2 & 69.6\,$\pm$\,3.9 & 45.7\,$\pm$\,2.0 & 25.6\,$\pm$\,5.0 & 33.1\,$\pm$\,3.7 & 77.7\,$\pm$\,2.3 & 37.5\,$\pm$\,2.9 & 39.8\,$\pm$\,13.0 \\
\SRPO{} & 25.5\,$\pm$\,1.2 & 6.7\,$\pm$\,2.7 & \textbf{55.2\,$\pm$\,0.7} & \textbf{74.9\,$\pm$\,0.2} & \textbf{46.2\,$\pm$\,0.9} & 24.5\,$\pm$\,3.8 & \textbf{33.9\,$\pm$\,2.3} & \textbf{80.6\,$\pm$\,0.9} & \textbf{39.3\,$\pm$\,2.6} & \textbf{45.5\,$\pm$\,11.8} \\
\midrule
\multicolumn{11}{l}{\textit{OLMo-3-7B-Instruct}} \\
\midrule
Base$^\dagger$ & 23.2 & 10.0 & 56.8 & 55.0 & 48.4 & 19.8 & 25.8 & 72.2 & 44.6 & 45.4 \\
GRPO & \textbf{26.6\,$\pm$\,0.2} & 11.1\,$\pm$\,5.7 & 67.1\,$\pm$\,1.2 & 62.0\,$\pm$\,4.6 & \textbf{59.7\,$\pm$\,0.8} & 20.7\,$\pm$\,0.4 & 26.3\,$\pm$\,2.1 & 72.3\,$\pm$\,0.6 & 51.3\,$\pm$\,2.6 & 42.3\,$\pm$\,5.4 \\
SCoRe & 17.3\,$\pm$\,1.5 & 0.0\,$\pm$\,0.0 & 43.4\,$\pm$\,6.9 & 60.2\,$\pm$\,2.8 & 33.9\,$\pm$\,4.2 & 18.9\,$\pm$\,1.1 & \textbf{29.6\,$\pm$\,3.7} & 71.5\,$\pm$\,1.1 & 32.6\,$\pm$\,5.9 & 29.4\,$\pm$\,3.8 \\
Cr-GRPO & 21.9\,$\pm$\,0.3 & 10.0\,$\pm$\,2.7 & 58.5\,$\pm$\,1.2 & 53.1\,$\pm$\,1.3 & 50.5\,$\pm$\,1.2 & 21.4\,$\pm$\,1.2 & 26.7\,$\pm$\,0.4 & 72.2\,$\pm$\,0.6 & 47.0\,$\pm$\,0.3 & 46.9\,$\pm$\,1.2 \\
SPO-Tree & 21.5\,$\pm$\,1.2 & 12.2\,$\pm$\,1.6 & 57.4\,$\pm$\,1.0 & 54.3\,$\pm$\,0.8 & 50.2\,$\pm$\,1.8 & 20.9\,$\pm$\,1.2 & 27.3\,$\pm$\,0.8 & 73.8\,$\pm$\,0.3 & 47.0\,$\pm$\,2.0 & 47.1\,$\pm$\,1.6 \\
\cmidrule(lr){1-11}
\RRPO{} & 26.1\,$\pm$\,0.9 & 10.0\,$\pm$\,4.7 & 66.2\,$\pm$\,6.0 & \textbf{68.1\,$\pm$\,0.1} & 59.0\,$\pm$\,5.0 & 20.5\,$\pm$\,5.4 & 28.0\,$\pm$\,4.0 & 56.9\,$\pm$\,25.7 & 37.9\,$\pm$\,18.9 & 33.9\,$\pm$\,12.1 \\
\SRPO{} & 24.8\,$\pm$\,1.9 & \textbf{15.6\,$\pm$\,4.2} & \textbf{67.5\,$\pm$\,3.5} & 66.6\,$\pm$\,2.5 & 59.6\,$\pm$\,2.6 & \textbf{24.7\,$\pm$\,1.4} & 28.7\,$\pm$\,0.6 & \textbf{74.8\,$\pm$\,1.0} & \textbf{55.9\,$\pm$\,0.6} & \textbf{54.5\,$\pm$\,2.0} \\
\bottomrule
\end{tabular}
\\[2pt]
{\footnotesize $^\dagger$Base is seed-invariant (single eval, no SD).}
\end{table*}

\subsection{Self-Guided Resets in Coding Tasks}
\label{sec:livecodebench}

We extend \SRPO{} to coding on LiveCodeBench~\citep{jain2024livecodebench} v6 medium (383 problems), holding the problem set fixed across training and evaluation and splitting the per-problem unit tests: a subset is used as the
training-time verifier (reward = fraction of train-split test cases passed), and the remainder are
held out for evaluation. We train for 1 epoch and compare \SRPO{}, \RRPO{}, and GRPO under a fixed
8-rollout-per-prompt budget (Figure~\ref{fig:val_curves_lcb_medium}).
SRPO's prefix mask concentrates each shared-prefix rollout's gradient signal onto its suffix tokens, the tokens that distinguish a correction from the failed seed. We measure this signal per token as the gradient magnitude with respect to the sampled token's logit, $g_{i,t} = (|\hat A_i|/T_i)(1 - \pi_\theta(y_{i,t} \mid \cdot))$, where $T_i$ is the rollout's active-token count: suffix length for shared-prefix, full response length for base (derivation in Appendix~\ref{app:grad_concentration}). Figure~\ref{fig:grad_tree_p2} colors each thought block by the mean of $g_{i,t}$ over its active tokens on one SRPO update applied to a LiveCodeBench-medium prompt; shared-prefix rollouts (top) deliver more per-token signal than base rollouts (bottom). Appendix~\ref{app:grad_concentration} confirms this throughout training: on prompts where both groups deliver gradient (some rollout has nonzero advantage in each), shared-prefix rollouts receive higher per-token signal at 10 of the 11 training steps. The fraction of prompts where both groups deliver gradient grows throughout training as the model gets better at self-correction: each shared-prefix rollout is conditioned on a failed seed trajectory, so the shared-prefix group starts with lower pass rates than the base group, but as the model improves at correcting from failed prefixes its shared-prefix pass rates rise and more prompts have both groups producing signal.

\begin{figure*}[t]
\centering
\begin{minipage}[c]{0.44\linewidth}
  \centering
  \includegraphics[width=\linewidth]{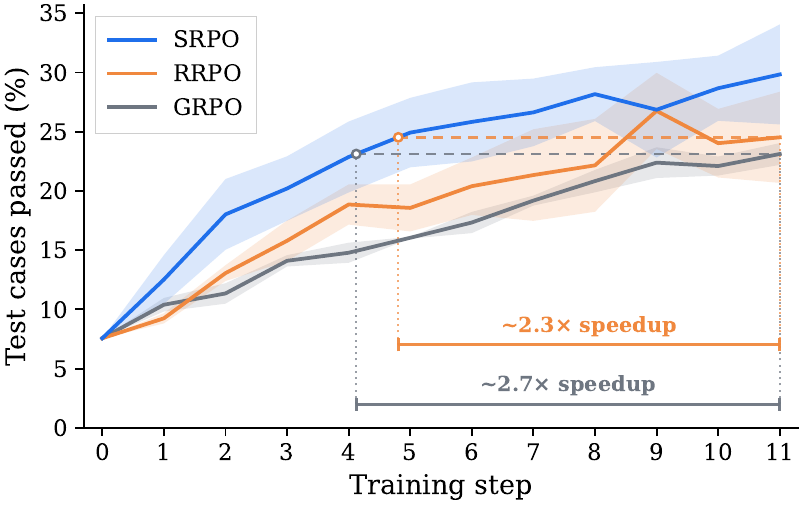}
\end{minipage}\hfill
\begin{minipage}[c]{0.56\linewidth}
  \centering
  \includegraphics[width=\linewidth]{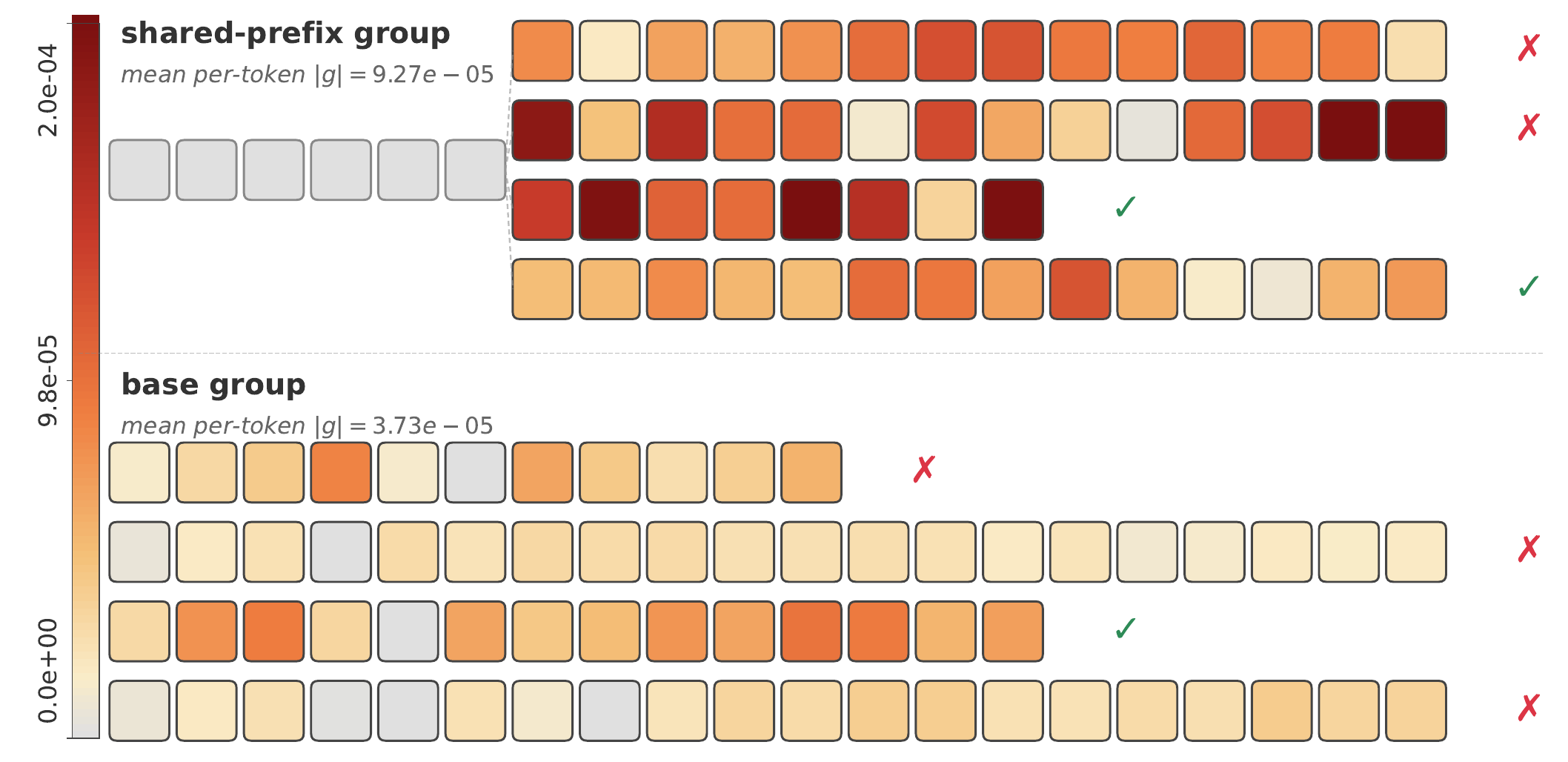}
\end{minipage}
\caption{\textbf{Left:} \SRPO{} converges to a higher pass rate and reaches
matching pass rates 2--3$\times$ faster than GRPO and \RRPO{}; \RRPO{} tracks
GRPO. Validation curves on LiveCodeBench v6 (medium) for OLMo-3-7B-Instruct.
Shaded bands are $\pm 1$ SE. Per-method training-compute breakdown in
Appendix~\ref{app:efficiency}.
\textbf{Right:} Per-thought mean of the per-token gradient signal on a single prompt under one \SRPO{} update (1$\times$4 split): 4 shared-prefix corrections (top) with their masked prefix and 4 base-group rollouts (bottom), illustrating where each loss assigns credit across a rollout's tokens. Gray = masked (no gradient).}
\label{fig:val_curves_lcb_medium}
\label{fig:grad_tree_p2}
\end{figure*}

\paragraph{Self-localization quality.}
\label{sec:self_loc}
We audit the trained \SRPO{} checkpoint's self-localization quality
(Figure~\ref{fig:loc_quality_lcb_medium}): for each reset-resample record we
compare the model's reset step against an oracle graded by Claude Opus 4.5,
and measure the downstream correction success of the four suffix rollouts in each shared-prefix group. Figure~\ref{fig:loc_quality_lcb_medium} reports four observations.
\textbf{(a)} On chains averaging $\approx 17$ thought steps (capped at
20), \SRPO{} concentrates resets in the early-middle of the reasoning
chain while \RRPO{} is approximately uniform, indicating that \SRPO{} is
actively localizing rather than resetting blindly. \textbf{(b)}
Roughly half of \SRPO{}'s localizations sit at or before Claude Opus 4.5's
failure step (clean prefix); the other half overshoot into the
erroneous region. \textbf{(c)} The correction rate decays
monotonically as the deviation from the oracle crosses from clean into erroneous.
\textbf{(d)} Pooling across deviations, clean (Self~$\le$~Oracle)
prefixes correct nearly 2$\times$ as often as erroneous ones (28.7\%
vs.\ 16.3\% Pass@4): getting the localization right is what makes the
reset productive.

\begin{figure*}[t]
\centering
\includegraphics[width=1.0\textwidth]{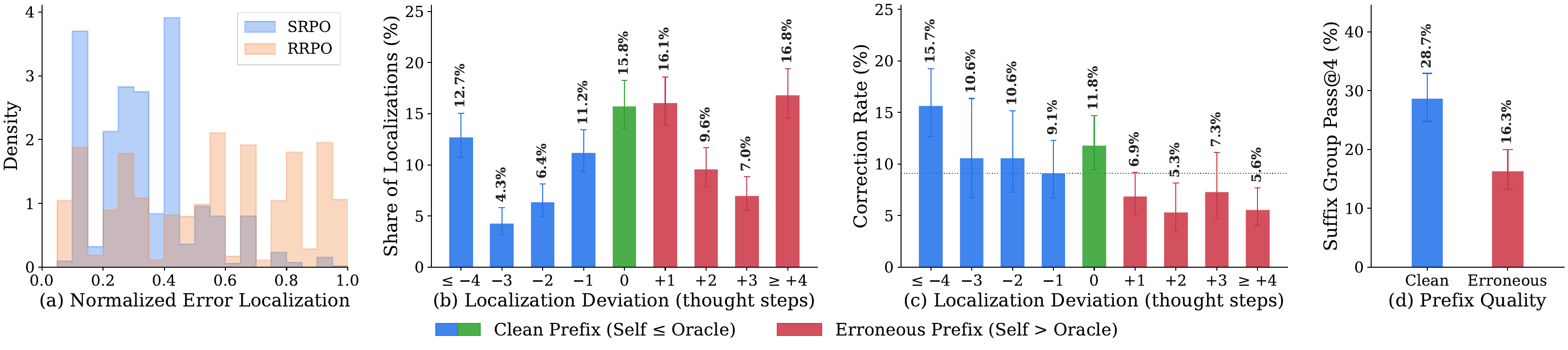}
\caption{Self-localization dynamics of OLMo-3-7B on coding tasks (LiveCodeBench v6 medium). (a) Self-localized vs.\ random reset distribution (average $\approx 17$ thoughts per sequence). (b) Self-localization vs.\ frontier-model localization distribution. (c) Correction rate of self-localization by deviation from oracle (Wilson 95\% CIs). (d) Pass@4 correction rate for clean vs.\ erroneous prefixes according to oracle.}
\label{fig:loc_quality_lcb_medium}
\end{figure*}

\begin{tcolorbox}[enhanced, colback=orange!5, colframe=orange!70, arc=4pt, boxrule=1pt, left=10pt, right=10pt, top=8pt, bottom=8pt, title={Takeaway}, fonttitle=\bfseries, coltitle=white, colbacktitle=orange!70, attach boxed title to top left={xshift=10pt, yshift=-8pt}, boxed title style={colback=orange!70, colframe=orange!70, arc=2pt, boxrule=0pt}]
Across our experiments, self-localized resets outperform RLVR post-training with random or no resets in both final performance and sample efficiency. Self-localization itself acts as an imperfect-but-effective approximation of the credit-assignment reset oracle.
\end{tcolorbox}

\section{Related Work}
\label{app:related_work}

Resets in RL have been studied primarily for exploration: \citet{mhammedi_power_2024} formalizes local simulator access, while DR-PO~\citep{chang_dataset_2024}, Go-Explore~\citep{ecoffet_first_2021}, and reverse curriculum generation~\citep{florensa_reverse_2018} reset to states drawn from preference data, novelty heuristics, or goal proximity. In each, the reset target is exogenous, chosen to broaden state coverage rather than to attribute credit for a failed outcome. Repurposing resets for credit assignment instead requires a step-level signal indicating \emph{which step} of a failed trajectory to reset to. Classically this comes from value functions or process reward models~\citep{lightman_lets_2023, uesato_solving_2022, wang_math-shepherd_2023, luo_improve_2024}, both requiring auxiliary trained models and either human annotation or a separate labeling pipeline. \citet{samanta2026structure} show that, given thought-level reasoning, the same model can self-localize the first erroneous thought, giving a training-free step-level signal that \SRPO{} uses to select the reset state.

The methods we compare against differ in how they use additional sampling to augment RL rollouts. Self-correction methods regenerate the full task conditioned on a prior attempt: SCoRe~\citep{kumar_training_2024} pairs each first attempt $y_1$ with a second attempt $y_2$ conditioned on it, shaping the reward by $R(y_2) - R(y_1)$ to drive correction without changing the first-attempt distribution. Critique-GRPO~\citep{zhang_critique-grpo_2025} augments $G$ iid chains with a critique-conditioned refinement of each, substitutes the best refinement into the buffer, and applies an off-policy correction on refinement tokens. \SRPO{} differs by resampling only a suffix on-policy from a verified-correct prefix, sidestepping the off-policy correction critique-conditioning requires and concentrating learning on the steps following the first error rather than re-deriving correct steps from scratch. InT~\citep{yang_int_2026} also targets a localized erroneous step but replaces it with a single counterfactual, and, unlike \SRPO{}, conditions on a human reference solution for both the localization and the counterfactual step, and uses SFT rather than RL. SPO-Tree~\citep{guo_segment_2025} similarly intermediate states by sampling alternative continuations from cutpoints, but they use heuristics (token counts or entropy spikes) rather than self-determined semantically meaningful thought boundaries, and they use a tree based approach with cutpoints at various points of the original rollout. Beyond these methods, ASTRO~\citep{kim_astro_2025} offers a complementary perspective on backtrack-and-resample frameworks: it uses tree search via MCTS to construct chains of thought that interleave backtracking and resampling within a single trajectory, which are then distilled into the model via SFT and refined with RL to teach it to perform such behaviors inline during generation. \SRPO{} instead leverages backtracking and resampling as a training-time mechanism to improve credit assignment, while the resulting model continues to produce direct solutions at inference.
\section{Limitations and Conclusion}
\label{sec:conclusion}

Our methods rely on the underlying language model for initial rollouts, self-localization, correction, and thought-level reasoning, which appear sufficient at the 7B--14B scale we evaluate. Within this regime, self-localization quality is the active bottleneck: clean prefixes correct nearly 2$\times$ as often as erroneous ones (Section~\ref{sec:self_loc}), and the gap to oracle-guided resets is set by localization accuracy. The methods also presume tasks that decompose into multi-step reasoning with clear reset states, and effectiveness on tasks lacking this structure is unclear. \RRPO{} and \SRPO{} further require verifiable rewards, and extending reset-based frameworks to non-verifiable settings~\citep{ouyang2022training, rafailov2023direct, jia2025writingzero, zhang2026chasing} remains an open direction. On the theory side, we establish the credit-assignment oracle's per-iteration improvement only for a single CPI step, similarly to~\citet{kakade2002cpi}, leaving a full convergence analysis of credit-aware policy optimization algorithms~\citep{shani2020adaptive, bhandari2021linear, agarwal2021theory, lan2023policy} as a next step.

Despite these limitations, this work establishes resets as a credit-assignment primitive for RL post-training of language models on sequential reasoning across math, science, strategic, commonsense, and coding benchmarks. We show that current language models already possess sufficient self-localization capability to drive meaningful improvements when used as a credit-assignment oracle. Two natural future directions follow. First, iterating on the reset mechanism itself, e.g.\ by training a dedicated localizer or jointly training the policy and the localizer, to close the gap to oracle-guided resets. Second, extending \SRPO{} to long-horizon tasks with several intermediate decision points, agentic tasks, and multi-turn dialogue. In these settings, focusing learning on key decision points via reset-based credit assignment becomes increasingly important.

\bibliographystyle{assets/plainnat}
\bibliography{paper}

\clearpage
\newpage
\beginappendix

\begin{tcolorbox}[enhanced, colback=gray!5, colframe=gray!55,
  arc=3pt, boxrule=0.6pt, left=10pt, right=10pt, top=8pt, bottom=8pt,
  title={\bfseries Appendix Contents}, fonttitle=\bfseries\small,
  coltitle=white, colbacktitle=gray!55,
  attach boxed title to top left={xshift=10pt, yshift=-8pt},
  boxed title style={colback=gray!55, colframe=gray!55, arc=2pt, boxrule=0pt}]
\small

\textbf{Related Work}
\begin{description}\setlength{\itemsep}{1pt}
\item[Section~\ref{app:related_work} Related Work.] Positioning relative to
prior credit-assignment, reset-based RL, and self-correction
methods, and the broader CPI / iterative-improvement lineage.
\end{description}

\textbf{Theory and Proofs}
\begin{description}\setlength{\itemsep}{1pt}
\item[Section~\ref{app:theory} Proofs and Tightness for Theorem~\ref{thm:main}.]
Self-contained derivation of \CPICARO{}'s per-iteration improvement
rate and sample complexity, with supporting lemmas
(Section~\ref{app:lemmas}), the proof itself (Section~\ref{app:proof-main}),
and a variance / signal-to-noise tightness analysis
(Section~\ref{app:tightness}).
\end{description}

\textbf{Sampler Implementation}
\begin{description}\setlength{\itemsep}{1pt}
\item[Section~\ref{app:rejection-sampling} Implementing the Credit Sampler via Rejection Sampling.]
How the abstract credit-sampler oracle in the algorithm is
realized as rejection sampling on the seed's prefix, and the
relationship to the self-localizer.
\end{description}

\textbf{Empirical}
\begin{description}\setlength{\itemsep}{1pt}
\item[Section~\ref{app:training_details} Training Details.] LoRA setup, no
clipped surrogate, and the full hyperparameter table
(optimization, schedule, rollout/sampling, hardware).
\item[Section~\ref{app:prompts} Prompts.] Verbatim prompts for the
thought-MDP rollout template and the L2 new self-localization method
revision.
\item[Section~\ref{app:loc_quality_raw} Self-Localization Quality: Raw Deviation.]
Companion to the main-text localization quality figure using raw
step-index deviation rather than the eff5 (5-token) filter.
\item[Section~\ref{app:sampling_full} Full Sampling-Strategy Ablation.]
$1\!\times\!4$ vs.\ $2\!\times\!4$ vs.\ $1\!\times\!8$ comparison
across both base models under a fixed 8-rollout budget.
\item[Section~\ref{app:grad_concentration} Per-token Gradient Concentration.]
Analytical bounds on within- and cross-group concentration from
suffix masking, with realized values from a held-out batch.
\item[Section~\ref{app:efficiency} Training Efficiency.] Per-method
compute breakdown on LiveCodeBench-medium (output tokens, wall
clock, generation time) for \SRPO{}, \RRPO{}, and GRPO.
\end{description}

\end{tcolorbox}
\clearpage

\section{Proofs and Tightness for Theorem~\ref{thm:main}}
\label{app:theory}

This appendix proves Theorem~\ref{thm:main} and establishes the
sharpness of its rates. The development is organized into three
subsections.

\paragraph{Appendix~\ref{app:lemmas}: Supporting lemmas.}
We state and prove the four lemmas the main proof relies on: the
advantage decomposition over $\mathcal{G}_{h,\tau} / \mathcal{G}_{h,\tau}^c$,
the credit-aware simulation lemma, the resulting
credit-aware CPI bound, and a greedy-policy transfer lemma.

\paragraph{Appendix~\ref{app:proof-main}: Proof of Theorem~\ref{thm:main}.}
We assemble these lemmas into a four-step argument: regression
error, $L^2$-to-advantage transfer, concentration of the advantage
estimator, and a case-specific CPI inequality. The random-sampler
and credit-sampler cases share all four steps and differ only in the
sampling distribution and which CPI bound is invoked.

\paragraph{Appendix~\ref{app:tightness}: Variance, signal-to-noise, and tightness.}
We bound the per-sample variance of the advantage estimator, give a
signal-to-noise reading that locates the $p_\pi^{\,2}$ factor in the
sample-complexity gap, and exhibit a finite-horizon MDP on which both
rates in Theorem~\ref{thm:main} are tight up to absolute constants.

Throughout, we work in the finite-horizon MDP
$(\mathcal{X}, \mathcal{Y}, P, r, H, \mu)$ of
Section~\ref{sec:preliminaries} and adopt the notation of
Section~\ref{sec:theory}: $\mathcal{G}_{h,\tau}$ is the $\tau$-improvable set at time $h$,
$p_{\pi, h} = \mathbb{P}_{x \sim d_\pi^h}[x \in \mathcal{G}_{h,\tau}]$,
$p_\pi = \frac{1}{H}\sum_{h=1}^H p_{\pi,h}$, and $\pi_\mathcal{G}$ is
the credit-assignment greedy policy that plays $\pi^+$ on
$\mathcal{G}_{h,\tau}$ and $\pi$ elsewhere. For any policy $\pi'$, recall
the conditional policy advantages
\[
  \mathbb{A}^{\mathcal{G}}_{\pi,\mu}(\pi')
  \;\coloneqq\; \mathbb{E}\!\left[A_h^\pi(x,y) \;\big|\; x \in \mathcal{G}_{h,\tau}\right],
  \qquad
  \mathbb{A}^{\mathcal{G}^c}_{\pi,\mu}(\pi')
  \;\coloneqq\; \mathbb{E}\!\left[A_h^\pi(x,y) \;\big|\; x \notin \mathcal{G}_{h,\tau}\right],
\]
with both expectations taken under
$h \sim \mathrm{Unif}[H]$, $x \sim d_\pi^h$, $y \sim \pi'_h(\cdot \mid x)$.

\subsection{Supporting Lemmas}
\label{app:lemmas}

\paragraph{Advantage decomposition.}

\begin{lemma}[Advantage decomposition]
\label{lem:advantage-decomp}
For any policy $\pi'$,
\begin{equation}
  \mathbb{A}_{\pi,\mu}(\pi')
  \;=\; p_\pi\, \mathbb{A}^{\mathcal{G}}_{\pi,\mu}(\pi')
  \;+\; (1 - p_\pi)\, \mathbb{A}^{\mathcal{G}^c}_{\pi,\mu}(\pi').
  \label{eq:advantage-decomp}
\end{equation}
For the credit-assignment greedy policy $\pi_\mathcal{G}$, the
$\mathcal{G}_{h,\tau}^c$ term vanishes and
\begin{equation}
  \mathbb{A}_{\pi,\mu}(\pi_\mathcal{G})
  \;=\; p_\pi\, \mathbb{A}^{\mathcal{G}}_{\pi,\mu}(\pi^+)
  \;\geq\; \tau\, p_\pi.
  \label{eq:advantage-decomp-piG}
\end{equation}
\end{lemma}

\begin{proof}
Equation~\eqref{eq:advantage-decomp} is the law of total expectation
applied to the partition $\{\mathcal{G}_{h,\tau}, \mathcal{G}_{h,\tau}^c\}$
under the joint distribution
$(h, x, y) \sim \mathrm{Unif}[H] \times d_\pi^h \times \pi'_h$.
For equation~\eqref{eq:advantage-decomp-piG}: on $\mathcal{G}_{h,\tau}^c$,
$(\pi_\mathcal{G})_h(\cdot \mid x) = \pi_h(\cdot \mid x)$, so
$\mathbb{E}_{y \sim (\pi_\mathcal{G})_h}[A_h^\pi(x, y)]
 = \mathbb{E}_{y \sim \pi_h}[A_h^\pi(x, y)] = 0$. On $\mathcal{G}_{h,\tau}$,
$(\pi_\mathcal{G})_h = \pi_h^+$ and
$\mathbb{E}_{y \sim \pi^+_h}[A_h^\pi(x, y)] = \max_y A_h^\pi(x, y) \geq \tau$
by definition of $\mathcal{G}_{h,\tau}$.
\end{proof}

\paragraph{Credit-aware simulation lemma.}
The classical CPI bound \citep{kakade2002cpi} controls
state-distribution drift between $\pi$ and a mixture
$\pi_\alpha = (1 - \alpha)\pi + \alpha\pi'$ by a quantity
$\propto \alpha H^2$, derived from the worst-case TV deviation
$\|\pi'_h - \pi_h\|_{\mathrm{TV}}(x) \leq 1$. When $\pi'$ agrees with
$\pi$ off $\mathcal{G}_{h,\tau}$, the same argument yields a tighter bound
that carries an extra $p_\pi$ factor.

\begin{lemma}[Credit-aware state-distribution simulation]
\label{lem:credit-sim}
Let $\pi'$ be any policy with
$\pi'_h(\cdot \mid x) = \pi_h(\cdot \mid x)$ for all $h \in [H]$ and
$x \notin \mathcal{G}_{h,\tau}$, and let
$\pi_\alpha = (1 - \alpha)\pi + \alpha\pi'$ for $\alpha \in [0, 1]$.
Then
\begin{equation}
  \sum_{h=1}^{H} \bigl\|d_{\pi_\alpha}^h - d_\pi^h\bigr\|_{\mathrm{TV}}
  \;\leq\; \alpha\, H^2\, p_\pi.
  \label{eq:credit-sim}
\end{equation}
\end{lemma}

\begin{proof}
We invoke the standard simulation lemma in the form
\begin{equation}
J_r(\pi_\alpha) - J_r(\pi)
\;=\; \alpha \sum_{h''=1}^H \mathbb{E}_{x \sim d_\pi^{h''}}\!\Bigl[\bigl(\pi'_{h''}(\cdot \mid x) - \pi_{h''}(\cdot \mid x)\bigr)^\top Q_{h''}^{\pi_\alpha, r}(x, \cdot)\Bigr],
\label{eq:pdl-alt}
\end{equation}
which follows from the performance-difference identity \citep{kakade2002cpi}
by recursively expanding $V_1^{\pi_\alpha} - V_1^\pi$ along trajectories
drawn under $\pi$.

For any state $s$ and time $h \in [H]$, the visitation
$d_\pi^h(s) = \mathbb{E}_\pi[\mathbf{1}\{x_h = s\}]$ is the expected
return of $\pi$ under the indicator reward $\mathbf{1}\{x_h = s\}$ that
fires once when the trajectory visits state $s$ at time $h$. The
corresponding Q-function is
\[
  Q_{h''}^{\pi_\alpha,\, \mathbf{1}\{x_h = s\}}(x, y) \;=\; \Pr_{\pi_\alpha}\!\bigl(x_h = s \mid x_{h''} = x,\, y_{h''} = y\bigr) \in [0, 1]
\]
for $h'' \leq h$ (and zero for $h'' > h$). Summing over $s$, the
trajectory passes through some state at time $h$ with probability one,
so
\begin{equation}
  \sum_{s \in \mathcal{X}} Q_{h''}^{\pi_\alpha,\, \mathbf{1}\{x_h = s\}}(x, y) \;=\; 1
  \qquad \text{for all } (x, y, h'') \text{ with } h'' \leq h.
  \label{eq:Q-sum}
\end{equation}

Apply equation~\eqref{eq:pdl-alt} with $r = \mathbf{1}\{x_h = s\}$. Writing
$\Delta_{h''}(\cdot \mid x) \coloneqq \pi'_{h''}(\cdot \mid x) - \pi_{h''}(\cdot \mid x)$,
taking absolute values, summing over $s$, and using $Q^{(s)} \geq 0$
together with equation~\eqref{eq:Q-sum},
\begin{align*}
  \sum_{s} \bigl|d_{\pi_\alpha}^h(s) - d_\pi^h(s)\bigr|
  &\;\leq\; \alpha \sum_{h''=1}^{h} \mathbb{E}_{x \sim d_\pi^{h''}}\!\biggl[\sum_{s}\Bigl|\sum_y \Delta_{h''}(y \mid x)\, Q_{h''}^{\pi_\alpha,\, \mathbf{1}\{x_h = s\}}(x, y)\Bigr|\biggr] \\
  &\;\leq\; \alpha \sum_{h''=1}^{h} \mathbb{E}_{x \sim d_\pi^{h''}}\!\biggl[\sum_y |\Delta_{h''}(y \mid x)|\, \underbrace{\sum_{s} Q_{h''}^{\pi_\alpha,\, \mathbf{1}\{x_h = s\}}(x, y)}_{=\, 1}\biggr] \\
  &\;=\; \alpha \sum_{h''=1}^{h} \mathbb{E}_{x \sim d_\pi^{h''}}\!\bigl[\|\Delta_{h''}(\cdot \mid x)\|_1\bigr].
\end{align*}

By assumption of the lemma, $\pi'_{h''}(\cdot \mid x) = \pi_{h''}(\cdot \mid x)$
for $x \notin \mathcal{G}_{h,\tau}$, hence
$\|\Delta_{h''}(\cdot \mid x)\|_1 \leq 2\, \mathbf{1}\{x \in \mathcal{G}_{h,\tau}\}$,
giving
\[
  \sum_{s} \bigl|d_{\pi_\alpha}^h(s) - d_\pi^h(s)\bigr|
  \;\leq\; 2\alpha \sum_{h''=1}^{h} p_{\pi, h''}.
\]

Since
$\|d_{\pi_\alpha}^h - d_\pi^h\|_{\mathrm{TV}} = \tfrac{1}{2}\sum_{s} |d_{\pi_\alpha}^h(s) - d_\pi^h(s)|$,
we obtain
$\|d_{\pi_\alpha}^h - d_\pi^h\|_{\mathrm{TV}} \leq \alpha \sum_{h''=1}^{h} p_{\pi, h''}$.
Summing over $h$,
\[
  \sum_{h=1}^{H} \bigl\|d_{\pi_\alpha}^h - d_\pi^h\bigr\|_{\mathrm{TV}}
  \;\leq\; \alpha \sum_{h=1}^{H} \sum_{h''=1}^{h} p_{\pi, h''}
  \;\leq\; \alpha H \sum_{h''=1}^{H} p_{\pi, h''}
  \;=\; \alpha H^2 p_\pi. \qedhere
\]
\end{proof}

\paragraph{Classical CPI improvement bound.}
For completeness and as a baseline for the credit-aware refinement
below, we state the classical CPI improvement bound
\citep{kakade2002cpi} in its finite-horizon form. We use this lemma
in the random-sampler case of the proof of Theorem~\ref{thm:main}.

\begin{lemma}[Classical CPI improvement bound; {\citealp{kakade2002cpi}}]
\label{lem:classical-cpi}
For any policy $\pi'$ and any $\alpha \in [0, 1]$, let
$\pi_\alpha = (1-\alpha)\pi + \alpha\pi'$ and
$\epsilon_{\mathrm{CPI}} \coloneqq \max_{h, x} \bigl|\mathbb{E}_{y \sim \pi'_h(\cdot \mid x)}[A_h^\pi(x, y)]\bigr|$.
Then
\begin{equation}
  J(\pi_\alpha) - J(\pi)
  \;\geq\; \alpha\, H\, \mathbb{A}_{\pi,\mu}(\pi')
  \;-\; \frac{\alpha^2\, H^2\, \epsilon_{\mathrm{CPI}}}{2},
  \qquad \epsilon_{\mathrm{CPI}} \leq H R_{\max}.
\end{equation}
\end{lemma}

\paragraph{Credit-aware CPI bound.}
Plugging Lemma~\ref{lem:credit-sim} into the performance-difference
identity yields a CPI-style improvement bound whose quadratic error
term carries an extra $p_\pi$ factor.

\begin{corollary}[Credit-aware CPI bound]
\label{cor:credit-cpi}
Under the conditions of Lemma~\ref{lem:credit-sim}, with
$\epsilon^\tau_{\mathrm{CPI}} \coloneqq \max_{h, x} \bigl|\mathbb{E}_{y \sim \pi'_h(\cdot \mid x)}[A_h^\pi(x, y)]\bigr|$,
\begin{equation}
  J(\pi_\alpha) - J(\pi)
  \;\geq\;
  \alpha\, H\, p_\pi\, \mathbb{A}^{\mathcal{G}}_{\pi,\mu}(\pi')
  \;-\; \alpha^2\, H^2\, p_\pi\, \epsilon^\tau_{\mathrm{CPI}}.
  \label{eq:credit-cpi}
\end{equation}
\end{corollary}

\begin{proof}
By the performance-difference lemma applied to $\pi_\alpha$ versus
$\pi$, and using $\mathbb{E}_{y \sim \pi_h(\cdot \mid x)}[A_h^\pi(x, y)] = 0$
together with the linearity of $\pi_\alpha$ in $\pi'$,
\[
  J(\pi_\alpha) - J(\pi)
  \;=\; \alpha \sum_{h=1}^H \mathbb{E}_{x \sim d_{\pi_\alpha}^h}\!\left[\bar A_h^\pi(x)\right],
  \qquad
  \bar A_h^\pi(x) \;\coloneqq\; \mathbb{E}_{y \sim \pi'_h(\cdot \mid x)}[A_h^\pi(x, y)].
\]
Splitting the expectation under $d_{\pi_\alpha}^h$ into its $d_\pi^h$
component plus an error,
\begin{align*}
  J(\pi_\alpha) - J(\pi)
  &= \alpha \sum_{h=1}^H \mathbb{E}_{d_\pi^h}[\bar A_h^\pi]
   \;+\; \alpha \sum_{h=1}^H \bigl(\mathbb{E}_{d_{\pi_\alpha}^h}[\bar A_h^\pi] - \mathbb{E}_{d_\pi^h}[\bar A_h^\pi]\bigr) \\
  &\geq \alpha\, H\, \mathbb{A}_{\pi,\mu}(\pi')
   \;-\; \alpha\, \epsilon^\tau_{\mathrm{CPI}}
        \sum_{h=1}^H \bigl\|d_{\pi_\alpha}^h - d_\pi^h\bigr\|_{\mathrm{TV}},
\end{align*}
where the first term uses the definition of $\mathbb{A}_{\pi,\mu}(\pi')$
and the second applies
$|\mathbb{E}_p f - \mathbb{E}_q f| \leq \|f\|_\infty\, \|p - q\|_{\mathrm{TV}}$
for non-negative $f$ (here $\bar A_h^\pi \geq 0$ since $\pi'_h$ agrees
with the greedy policy on $\mathcal{G}_{h,\tau}$ and with $\pi_h$
elsewhere, so $\bar A_h^\pi(x) = \max_y A_h^\pi(x, y) \geq 0$ on
$\mathcal{G}_{h,\tau}$ and $\bar A_h^\pi(x) = 0$ on $\mathcal{G}_{h,\tau}^c$).
By Lemma~\ref{lem:advantage-decomp},
$\mathbb{A}_{\pi,\mu}(\pi') = p_\pi\, \mathbb{A}^{\mathcal{G}}_{\pi,\mu}(\pi')$ since
$\pi'$ agrees with $\pi$ off $\mathcal{G}_{h,\tau}$. Substituting
Lemma~\ref{lem:credit-sim} into the error term yields
equation~\eqref{eq:credit-cpi}.
\end{proof}

\paragraph{Greedy-policy transfer.}
The algorithm uses the empirical greedy policy
$\hat\pi^+_h(x) \coloneqq \argmax_y \hat Q(x, y, h)$ in place of the
true greedy $\pi^+_h$. The next lemma controls the loss from this
substitution.

\begin{lemma}[Greedy-policy transfer]
\label{lem:greedy-transfer}
Let $\hat Q : \mathcal{X} \times \mathcal{Y} \times [H] \to \mathbb{R}$
be any function and define
$\hat\pi^+_h(x) \coloneqq \argmax_{y \in \mathcal{Y}} \hat Q(x, y, h)$.
For every $h \in [H]$ and $x \in \mathcal{X}$,
\begin{equation}
  A_h^\pi\bigl(x, \hat\pi^+_h(x)\bigr)
  \;\geq\; A_h^\pi\bigl(x, \pi^+_h(x)\bigr)
  \;-\; 2\, \sup_{y \in \mathcal{Y}} \bigl|\hat Q(x, y, h) - Q_h^\pi(x, y)\bigr|.
  \label{eq:greedy-transfer}
\end{equation}
\end{lemma}

\begin{proof}
Let $\Delta(y) \coloneqq |\hat Q(x, y, h) - Q_h^\pi(x, y)|$ and write
$y^+ = \pi^+_h(x)$, $\hat y^+ = \hat\pi^+_h(x)$. By the empirical
greedy property, $\hat Q(x, \hat y^+, h) \geq \hat Q(x, y^+, h)$, so
\begin{align*}
  Q_h^\pi(x, \hat y^+)
  &\geq \hat Q(x, \hat y^+, h) - \Delta(\hat y^+) \\
  &\geq \hat Q(x, y^+, h) - \Delta(\hat y^+) \\
  &\geq Q_h^\pi(x, y^+) - \Delta(y^+) - \Delta(\hat y^+) \\
  &\geq Q_h^\pi(x, y^+) - 2\sup_{y} \Delta(y).
\end{align*}
Subtracting $V_h^\pi(x)$ from both sides yields equation~\eqref{eq:greedy-transfer}.
\end{proof}

\paragraph{Least-squares regression error.}
The proof of Theorem~\ref{thm:main} fits $\hat Q$ by least-squares
regression on the Q-rollout targets. Under realizability and a
finite hypothesis class, the resulting error obeys a fast rate that
is standard in the empirical-risk-minimization literature
\citep{geer2000empirical}.

\begin{lemma}[Least-squares regression error for finite, realizable classes]
\label{lem:regression}
Let $\mathcal{F}$ be a finite class of functions
$\mathcal{X} \times \mathcal{Y} \times [H] \to [0, H R_{\max}]$ and
suppose $Q_h^\pi \in \mathcal{F}$ for every $h \in [H]$
(\emph{realizability}). Let $d$ be any distribution over
$(x, y, h)$, let
$\{(x_i, y_i, h_i)\}_{i=1}^n \stackrel{\mathrm{i.i.d.}}{\sim} d$, and
let $\hat Q_i$ be conditionally unbiased estimates of $Q_{h_i}^\pi(x_i, y_i)$
with $\hat Q_i \in [0, H R_{\max}]$ a.s. Define
$\hat Q \coloneqq \argmin_{f \in \mathcal{F}} \sum_{i=1}^n (f(x_i, y_i, h_i) - \hat Q_i)^2$.
For any $\delta \in (0, 1)$, with probability at least $1 - \delta$,
\begin{equation}
  \mathbb{E}_{(x, y, h) \sim d}\!\left[\bigl(\hat Q(x, y, h) - Q_h^\pi(x, y)\bigr)^2\right]
  \;\leq\; \frac{C_0\, H^2 R_{\max}^2 \log(|\mathcal{F}|/\delta)}{n},
  \label{eq:regression-bound}
\end{equation}
for an absolute constant $C_0$.
\end{lemma}

\begin{proof}[Sketch]
This is the standard fast-rate bound for ERM over a finite,
realizable function class with bounded targets, obtained by
combining a Bernstein-type concentration inequality applied to
$(f(x, y, h) - \hat Q)^2 - (Q_h^\pi(x, y) - \hat Q)^2$ for each
fixed $f \in \mathcal{F}$ with a union bound over $\mathcal{F}$;
see, e.g., \citet{geer2000empirical}. The
realizability hypothesis $Q_h^\pi \in \mathcal{F}$ removes the
approximation error and yields the $1/n$ (rather than
$1/\sqrt{n}$) fast rate.
\end{proof}

\subsection{Proof of Theorem~\ref{thm:main}}
\label{app:proof-main}

We prove the two cases of Theorem~\ref{thm:main} via a common
pipeline: (i) least-squares regression error, (ii) transfer from
$L^2$-error on $\hat Q$ to advantage error of $\hat\pi^+$, and (iii)
plugging the resulting advantage lower bound into the appropriate CPI
inequality. The cases differ only in the sampling distribution and
which CPI inequality is applied.

\paragraph{Setup.}
Let $d$ denote the sampling distribution over $(x, y, h)$:
$d \,=\, d_\pi^h \times \mathrm{Unif}(\mathcal{Y}) \times \mathrm{Unif}[H]$
for the random sampler (\CPIRR{}), and
$d \,=\, (d_\pi^h \mid x \in \mathcal{G}_{h,\tau}) \times \mathrm{Unif}(\mathcal{Y}) \times \mathrm{Unif}[H]$
for the credit sampler (\CPICARO{}). The Q-rollout
$\textsc{QRollout}(\pi, x_i, y_i, h_i)$ returns $\hat Q_i$ with
$\mathbb{E}[\hat Q_i \mid x_i, y_i, h_i] = Q_{h_i}^\pi(x_i, y_i)$ and
$\hat Q_i \in [0, H R_{\max}]$ (per-step rewards lie in $[0, R_{\max}]$
and are non-negative). For any policy $\pi'$, define the
\emph{policy advantage under $d$} as
\[
  \mathbb{A}_d(\pi')
  \;\coloneqq\;
  \mathbb{E}_{(x, h) \sim d_{x, h}}
  \mathbb{E}_{y \sim \pi'_h(\cdot \mid x)}\!\bigl[A_h^\pi(x, y)\bigr],
\]
where $d_{x, h}$ denotes the marginal of $d$ on $(x, h)$. In particular,
\begin{align*}
  \mathbb{A}_d(\pi') &\;=\; \mathbb{A}_{\pi,\mu}(\pi')
   && \text{(random sampler),} \\
  \mathbb{A}_d(\pi') &\;=\; \mathbb{A}^{\mathcal{G}}_{\pi,\mu}(\pi')
   && \text{(credit sampler).}
\end{align*}
This unified notation lets Steps~1--3 below be stated once for both samplers, with the two cases distinguished only in Step~4 via the identities above.

The algorithm uses the empirical greedy policy
$\hat\pi^+_h(x) = \arg\max_y \hat Q(x, y, h)$ derived from the fitted
$\hat Q$ rather than the true greedy $\pi^+_h$. The proof tracks this
substitution: Step~2 bounds the resulting advantage error via
Lemma~\ref{lem:greedy-transfer}, Step~3 controls concentration of
$\hat{\mathbb{A}}$ around $\mathbb{A}_d(\hat\pi^+)$ (not
$\mathbb{A}_d(\pi^+)$), and Step~4 plugs $\hat\pi^+$
(or $\hat\pi_\mathcal{G}$) into the appropriate CPI bound.

\paragraph{Step~1: Least-squares regression error.}
Lemma~\ref{lem:regression} applied with confidence parameter
$\delta/2$ gives, with probability at least $1 - \delta/2$,
\begin{equation}
  \mathbb{E}_{(x, y, h) \sim d}\!\left[\bigl(\hat Q(x, y, h) - Q_h^\pi(x, y)\bigr)^2\right]
  \;\leq\; \hat\epsilon^2,
  \qquad
  \hat\epsilon^2 \;\coloneqq\; \frac{C_0\, H^2 R_{\max}^2 \log(2|\mathcal{F}|/\delta)}{n}.
  \label{eq:regression}
\end{equation}

\paragraph{Step~2: From $L^2$ error to advantage error.}
Lemma~\ref{lem:greedy-transfer} gives, pointwise in $(x, h)$,
\[
  A_h^\pi\bigl(x, \pi^+_h(x)\bigr) - A_h^\pi\bigl(x, \hat\pi^+_h(x)\bigr)
  \;\leq\; 2\, \sup_{y \in \mathcal{Y}} \bigl|\hat Q(x, y, h) - Q_h^\pi(x, y)\bigr|.
\]
Taking expectations under $d$,
\begin{equation}
  \bigl|\mathbb{A}_d(\hat\pi^+) - \mathbb{A}_d(\pi^+)\bigr|
  \;\leq\;
  2\, \mathbb{E}_{(x, h) \sim d}\!\left[\sup_{y} |\hat Q(x, y, h) - Q_h^\pi(x, y)|\right]
  \;\leq\; 2\sqrt{|\mathcal{Y}|}\, \hat\epsilon,
  \label{eq:l2-to-advantage}
\end{equation}
where $\mathbb{A}_d(\pi')$ denotes the policy advantage under sampling distribution $d$. To obtain the last step, set $\Delta(x, y, h) \coloneqq |\hat Q(x, y, h) - Q_h^\pi(x, y)|$ and let $d_{x,h}$ denote the marginal of $d$ on $(x, h)$. Since $y$ is uniform on $\mathcal{Y}$ under $d$ conditional on $(x, h)$, for each fixed $(x, h)$ the sup-$L^2$ inequality on the finite action set gives
\[
  \sup_{y \in \mathcal{Y}} \Delta(x, y, h)
  \;\leq\;
  \sqrt{|\mathcal{Y}|\, \mathbb{E}_{y \sim \mathrm{Unif}(\mathcal{Y})}\!\bigl[\Delta(x, y, h)^2\bigr]}.
\]
Taking expectation over $(x, h) \sim d_{x, h}$ and applying Jensen's inequality,
\[
  \mathbb{E}_{(x, h) \sim d_{x, h}}\!\bigl[\sup_y \Delta(x, y, h)\bigr]
  \;\leq\;
  \sqrt{|\mathcal{Y}|\, \mathbb{E}_{(x, y, h) \sim d}\!\bigl[\Delta(x, y, h)^2\bigr]}
  \;\leq\;
  \sqrt{|\mathcal{Y}|}\, \hat\epsilon,
\]
where the last inequality is equation~\eqref{eq:regression}.

\paragraph{Step~3: Concentration of $\hat{\mathbb{A}}$.}
The algorithm's plug-in advantage estimate is
$\hat{\mathbb{A}} = \frac{1}{n} \sum_{i=1}^n |\mathcal{Y}|\bigl(\hat\pi^+_{h_i}(y_i \mid x_i) - \pi_{h_i}(y_i \mid x_i)\bigr) \hat Q(x_i, y_i, h_i)$.
The empirical greedy $\hat\pi^+$ is data-dependent, so a direct
Hoeffding bound on $\hat{\mathbb{A}}$ for a fixed policy does not
apply. Instead, observe that
$\hat\pi^+ \in \Pi_\mathcal{F} \coloneqq \{\arg\max_y f(\cdot, \cdot, \cdot) : f \in \mathcal{F}\}$,
so $|\Pi_\mathcal{F}| \leq |\mathcal{F}|$. Applying Hoeffding's
inequality to each fixed $\pi' \in \Pi_\mathcal{F}$ together with
the variance bound of Lemma~\ref{lem:variance}, and union-bounding
over $\Pi_\mathcal{F}$, with probability at least $1 - \delta/2$,
\begin{equation}
  \sup_{\pi' \in \Pi_\mathcal{F}}
  \bigl|\hat{\mathbb{A}}(\pi') - \mathbb{A}_d(\pi')\bigr|
  \;\leq\; C_1\, H R_{\max}\, \sqrt{\frac{|\mathcal{Y}|\, \log(2|\mathcal{F}|/\delta)}{n}}
  \;\leq\; \sqrt{|\mathcal{Y}|}\, \hat\epsilon,
  \label{eq:advantage-concentration}
\end{equation}
where the absolute constant $C_1 \leq \sqrt{C_0}$ has been absorbed
into $\hat\epsilon$ defined in equation~\eqref{eq:regression}.
Specializing to $\pi' = \hat\pi^+$ and combining with
equation~\eqref{eq:l2-to-advantage},
\begin{equation}
  \bigl|\hat{\mathbb{A}} - \mathbb{A}_d(\pi^+)\bigr|
  \;\leq\; 3\sqrt{|\mathcal{Y}|}\, \hat\epsilon.
  \label{eq:hatA-vs-Apiplus}
\end{equation}

\paragraph{Step~4: Apply the appropriate CPI bound.}

\subparagraph*{\textcolor{maroon}{Random sampler (\CPIRR{}).}}
Here $d = d_\pi^h \times \mathrm{Unif}(\mathcal{Y}) \times \mathrm{Unif}[H]$,
so $\mathbb{A}_d(\pi^+) = \mathbb{A}_{\pi,\mu}(\pi^+) \geq \tau\, p_\pi$ by
Lemma~\ref{lem:advantage-decomp} (specifically,
$\mathbb{A}_{\pi,\mu}(\pi^+) \geq \mathbb{A}_{\pi,\mu}(\pi_\mathcal{G}) \geq \tau p_\pi$).
Choose $n$ so that $3\sqrt{|\mathcal{Y}|}\, \hat\epsilon \leq \tau p_\pi/2$,
i.e.,
$\hat\epsilon^2 \leq \tau^2 p_\pi^2 / (36 |\mathcal{Y}|)$. By
equation~\eqref{eq:regression}, this is implied by
\[
  n \;\geq\; \frac{C\, |\mathcal{Y}|\, H^2 R_{\max}^2 \log(2|\mathcal{F}|/\delta)}{\tau^2\, p_\pi^2}
\]
for a universal constant $C = 36 C_0$. On the resulting $1 - \delta$
event, two lower bounds hold. From
equation~\eqref{eq:hatA-vs-Apiplus} and
$\mathbb{A}_d(\pi^+) = \mathbb{A}_{\pi,\mu}(\pi^+) \geq \tau p_\pi$,
\[
  \hat{\mathbb{A}}
  \;\geq\; \mathbb{A}_d(\pi^+) - 3\sqrt{|\mathcal{Y}|}\, \hat\epsilon
  \;\geq\; \tau p_\pi - \tau p_\pi / 2
  \;=\; \tau p_\pi / 2.
\]
From equation~\eqref{eq:l2-to-advantage} and the same identity,
\[
  \mathbb{A}_{\pi,\mu}(\hat\pi^+)
  \;=\; \mathbb{A}_d(\hat\pi^+)
  \;\geq\; \mathbb{A}_d(\pi^+) - 2\sqrt{|\mathcal{Y}|}\, \hat\epsilon
  \;\geq\; \tau p_\pi - \tau p_\pi / 3
  \;\geq\; \tau p_\pi / 2.
\]
The algorithm returns
$\pi_{\hat\alpha} = (1 - \hat\alpha)\pi + \hat\alpha\, \hat\pi^+$, the
convex combination of $\pi$ with the empirical greedy $\hat\pi^+$.
Applying Lemma~\ref{lem:classical-cpi} with $\pi' = \hat\pi^+$ and
the algorithm's choice
$\hat\alpha = \min\{1, \hat{\mathbb{A}} / (H^2 R_{\max})\}$ and using
$\epsilon_{\mathrm{CPI}} \leq H R_{\max}$,
\begin{align*}
  J(\pi_{\hat\alpha}) - J(\pi)
  &\;\geq\; \frac{\hat{\mathbb{A}}\bigl(2\, \mathbb{A}_{\pi,\mu}(\hat\pi^+) - \hat{\mathbb{A}}\bigr)}{2 H R_{\max}} \\
  &\;\geq\; \frac{(\tau p_\pi/2)(\tau p_\pi/2)}{2 H R_{\max}}
  \;=\; \frac{\tau^2\, p_\pi^2}{8 H R_{\max}},
\end{align*}
where in the second line we used
$\hat{\mathbb{A}} \geq \tau p_\pi / 2$ and
$2\, \mathbb{A}_{\pi,\mu}(\hat\pi^+) - \hat{\mathbb{A}} \geq \tau p_\pi/2$
(applying equation~\eqref{eq:hatA-vs-Apiplus} a second time:
$\hat{\mathbb{A}} \leq \mathbb{A}_{\pi,\mu}(\hat\pi^+) + \tau p_\pi/2
 \leq 2\, \mathbb{A}_{\pi,\mu}(\hat\pi^+) - \tau p_\pi/2$
when $\mathbb{A}_{\pi,\mu}(\hat\pi^+) \geq \tau p_\pi$).

\subparagraph*{\textcolor{blue}{Credit sampler (\CPICARO{}).}}
Here $d = (d_\pi^h \mid \mathcal{G}_{h,\tau}) \times \mathrm{Unif}(\mathcal{Y}) \times \mathrm{Unif}[H]$,
so $\mathbb{A}_d(\pi^+) = \mathbb{A}^{\mathcal{G}}_{\pi,\mu}(\pi^+) \geq \tau$.
Choose $n$ so that $3\sqrt{|\mathcal{Y}|}\, \hat\epsilon \leq \tau/2$,
i.e.,
\[
  n \;\geq\; \frac{C\, |\mathcal{Y}|\, H^2 R_{\max}^2 \log(2|\mathcal{F}|/\delta)}{\tau^2}.
\]
On the resulting $1 - \delta$ event, two lower bounds hold. From
equation~\eqref{eq:hatA-vs-Apiplus} and
$\mathbb{A}_d(\pi^+) = \mathbb{A}^{\mathcal{G}}_{\pi,\mu}(\pi^+) \geq \tau$,
\[
  \hat{\mathbb{A}}
  \;\geq\; \mathbb{A}_d(\pi^+) - 3\sqrt{|\mathcal{Y}|}\, \hat\epsilon
  \;\geq\; \tau - \tau/2
  \;=\; \tau/2.
\]
From equation~\eqref{eq:l2-to-advantage} and the same identity,
\[
  \mathbb{A}^{\mathcal{G}}_{\pi,\mu}(\hat\pi^+)
  \;=\; \mathbb{A}_d(\hat\pi^+)
  \;\geq\; \mathbb{A}_d(\pi^+) - 2\sqrt{|\mathcal{Y}|}\, \hat\epsilon
  \;\geq\; \tau - \tau/3
  \;\geq\; \tau/2.
\]
The algorithm returns
$\pi_{\hat\alpha} = (1 - \hat\alpha)\pi + \hat\alpha\, \hat\pi_\mathcal{G}$
where $\hat\pi_\mathcal{G}$ agrees with $\hat\pi^+$ on $\mathcal{G}_{h,\tau}$
and with $\pi$ elsewhere, so it satisfies the hypothesis of
Lemma~\ref{lem:credit-sim}. Applying
Corollary~\ref{cor:credit-cpi} with $\pi' = \hat\pi_\mathcal{G}$,
$\epsilon^\tau_{\mathrm{CPI}} \leq H R_{\max}$, and the algorithm's
$\hat\alpha = \min\{1, \hat{\mathbb{A}}/(2 H^2 R_{\max})\}$,
\begin{align*}
  J(\pi_{\hat\alpha}) - J(\pi)
  &\;\geq\; \hat\alpha\, H\, p_\pi\, \mathbb{A}^{\mathcal{G}}_{\pi,\mu}(\hat\pi^+)
   \;-\; \hat\alpha^2\, H^2\, p_\pi\, \epsilon^\tau_{\mathrm{CPI}} \\
  &\;\geq\; \frac{p_\pi\, \hat{\mathbb{A}}\bigl(2\, \mathbb{A}^{\mathcal{G}}_{\pi,\mu}(\hat\pi^+) - \hat{\mathbb{A}}\bigr)}{4 H R_{\max}}
  \;\geq\; \frac{p_\pi\, (\tau/2)(\tau/2)}{4 H R_{\max}}
  \;=\; \frac{\tau^2\, p_\pi}{16 H R_{\max}},
\end{align*}
where the second inequality plugs in $\hat\alpha$ and uses
$\epsilon^\tau_{\mathrm{CPI}} \leq H R_{\max}$, and the third uses
$\hat{\mathbb{A}} \geq \tau/2$ and
$2\, \mathbb{A}^{\mathcal{G}}_{\pi,\mu}(\hat\pi^+) - \hat{\mathbb{A}} \geq \tau/2$
(by the same argument as in the random case, with $\tau p_\pi$
replaced by $\tau$).

A union bound over the $1 - \delta/2$ events
in equations~\eqref{eq:regression} and~\eqref{eq:advantage-concentration}
yields the stated $1 - \delta$ probability in both cases. \qed

\subsection{Variance, Signal-to-Noise, and Tightness}
\label{app:tightness}

The central result of this subsection is Proposition~\ref{prop:tight}, which shows that the random-sampler's empirical advantage estimator requires $\Omega(|\mathcal{Y}|\, R_{\max}^2 / (\tau^2 p_\pi^{\,2}))$ samples to certify a non-trivial lower bound, matching the upper bound for \CPIRR{} in Theorem~\ref{thm:main}. This shows the $1/p_\pi^{\,2}$ separation between \CPIRR{} and \CPICARO{} is a genuine algorithmic improvement --- the credit-assignment oracle removes a $p_\pi^{\,2}$ factor that the on-policy random-reset estimator provably cannot.

\paragraph{Variance.}
For both samplers, the per-sample term
$Y_i \coloneqq |\mathcal{Y}|\bigl(\hat\pi^+_{h_i}(y_i \mid x_i) - \pi_{h_i}(y_i \mid x_i)\bigr) \hat Q(x_i, y_i, h_i)$
of the advantage estimator $\hat{\mathbb{A}} = \tfrac{1}{n}\sum_i Y_i$
is bounded by $|\mathcal{Y}| H R_{\max}$ almost surely, and its
variance is bounded uniformly by
$\sigma^2 \coloneqq 2 |\mathcal{Y}| H^2 R_{\max}^2$.

\begin{lemma}[Variance of $Y_i$]
\label{lem:variance}
For both the random sampler and the credit sampler,
$\mathrm{Var}(Y_i) \leq \mathbb{E}[Y_i^2] \leq 2 |\mathcal{Y}| H^2 R_{\max}^2$.
\end{lemma}

\begin{proof}
Since $\hat Q_i \in [0, H R_{\max}]$, $\hat Q$ takes values in
$[0, H R_{\max}]$ as well, so
$\mathbb{E}[Y_i^2] \leq |\mathcal{Y}|^2 H^2 R_{\max}^2 \cdot \mathbb{E}_{(x, y, h) \sim d}[(\hat\pi^+_h(y \mid x) - \pi_h(y \mid x))^2]$.
For each fixed $(x, h)$, the inner expectation under
$y \sim \mathrm{Unif}(\mathcal{Y})$ is at most $2/|\mathcal{Y}|$
(the sum equals
$(1 - \pi_h(\hat\pi^+_h(x) \mid x))^2 + \sum_{y \neq \hat\pi^+_h(x)} \pi_h(y \mid x)^2 \leq 2$
since $\hat\pi^+_h$ is deterministic). Combining yields
$\mathbb{E}[Y_i^2] \leq 2 |\mathcal{Y}| H^2 R_{\max}^2$, uniformly in
the conditioning event.
\end{proof}

\paragraph{Signal-to-noise.}
\label{subsec:s2n}
The two samplers share the same noise $\sigma^2$, so the
sample-complexity separation between \CPIRR{} and \CPICARO{} in
Theorem~\ref{thm:main} comes from the \emph{signal}: the
magnitude of the target each estimator pursues. \CPIRR{} estimates
$\mathbb{A}_{\pi,\mu}(\hat\pi^+) \geq \tau p_\pi$, so certifying
$\hat{\mathbb{A}} \geq \tau p_\pi/2$ requires precision
$\Theta(\tau p_\pi)$ and hence
$n \gtrsim \sigma^2/(\tau p_\pi)^2 \propto |\mathcal{Y}| H^2 R_{\max}^2/(\tau^2 p_\pi^{\,2})$
samples. \CPICARO{} estimates the conditional advantage
$\mathbb{A}^{\mathcal{G}}_{\pi,\mu}(\hat\pi^+) \geq \tau$ directly on
$\mathcal{G}_{h,\tau}$, so precision $\Theta(\tau)$ suffices and
$n \gtrsim \sigma^2/\tau^2 \propto |\mathcal{Y}| H^2 R_{\max}^2/\tau^2$,
\emph{independent of $p_\pi$}. Same noise, $1/p_\pi$-larger signal:
the $p_\pi^{\,2}$ factor in \CPIRR{}'s sample complexity is the
squared signal-to-noise ratio of the two targets, not a variance
gap. The next subsection makes this $1/p_\pi^{\,2}$ tightness precise
via an explicit single-step ($H = 1$) construction.

\paragraph{Formalizing the signal-to-noise gap.}
We exhibit a single-step ($H = 1$) MDP on which the random-sampler's
empirical-mean advantage estimator $\hat{\mathbb{A}}$ requires
$\Omega(|\mathcal{Y}|\, R_{\max}^2 / (\tau^2 p_\pi^{\,2}))$ samples to
certify a non-trivially positive lower bound, matching \CPIRR{}'s rate
in Theorem~\ref{thm:main}. This shows the $1/p_\pi^{\,2}$ factor is
fundamental to \CPIRR{} under empirical-mean estimation with only
$\tau, p_\pi$ information --- the credit-assignment oracle removes a
$p_\pi^{\,2}$ factor that this estimator provably cannot.

The argument is finite-sample anti-concentration for empirical means,
in the spirit of Cramer's theorem for sums of bounded random
variables~\citep{dembo_zeitouni_2010}, made explicit via Berry-Esseen
(Lemma~\ref{lem:berry-esseen}). We work at $H = 1$ since the $H^2$
factor is shared by \CPIRR{} and \CPICARO{}, and our goal is to isolate
the $1/p_\pi^{\,2}$ separation.

\begin{lemma}[Berry-Esseen inequality, iid case; {\citealp[Theorems~1 and~2]{shevtsova2010}}]
\label{lem:berry-esseen}
Let $Y_1, \ldots, Y_n$ be iid real-valued random variables with mean
$\mu_Y \coloneqq \mathbb{E}[Y_i]$, variance
$\sigma_Y^2 \coloneqq \mathrm{Var}(Y_i) > 0$, and finite third absolute
moment $\rho \coloneqq \mathbb{E}|Y_i - \mu_Y|^3$. Let
$\bar Y_n \coloneqq \frac{1}{n}\sum_{i=1}^n Y_i$ and let $\Phi$ denote
the standard normal CDF. Then
\[
  \sup_{t \in \mathbb{R}} \Bigl|\Pr\!\bigl((\bar Y_n - \mu_Y)/(\sigma_Y/\sqrt{n}) \leq t\bigr) - \Phi(t)\Bigr|
  \;\leq\; \frac{0.6\, \rho}{\sigma_Y^3 \sqrt{n}}.
\]
\end{lemma}

\noindent\emph{Construction.} Fix $|\mathcal{Y}| \geq 2$ and let $H = 1$.
The state space has two states $\{x_1, x_2\}$, with initial
distribution $\mu(x_1) = p_\pi$ and $\mu(x_2) = 1 - p_\pi$. The action
set is $\{y_0, y_1, \ldots, y_{|\mathcal{Y}| - 1}\}$, and the base policy
plays $\pi(y_0 \mid x) = 1$ for all $x$. Fix $\tau \leq R_{\max}/2$ and
$\varepsilon \ll \tau\, p_\pi / (1 - p_\pi)$. Rewards have a constant
baseline $R_{\max}/2$ for every action other than $y_1$, and a small
state-dependent excess for $y_1$:
\[
  r(x, y_k) = R_{\max}/2
    \;\;\text{for } k \neq 1, \;\forall x,
  \qquad
  r(x_1, y_1) = R_{\max}/2 + \tau,
  \qquad
  r(x_2, y_1) = R_{\max}/2 + \varepsilon.
\]
All rewards lie in $[0, R_{\max}]$.

\noindent\emph{Resulting structure.} Under $\pi \equiv y_0$,
$V^\pi(x) = R_{\max}/2$ for both $x$, and the advantages are
$A^\pi(x_1, y_1) = \tau$, $A^\pi(x_2, y_1) = \varepsilon$, and
$A^\pi(x, y_k) = 0$ for $k \neq 1$. At threshold $\tau$,
$\mathcal{G}_{h,\tau} = \{x_1\}$ with on-policy probability $p_\pi$,
and the greedy policy $\pi^+$ plays $y_1$ everywhere. The expected
advantage of $\pi^+$ against $\pi$ is therefore
\[
  \mathbb{A}_{\pi,\mu}(\pi^+) \;=\; p_\pi\, \tau + (1 - p_\pi)\, \varepsilon \;=\; \Theta(\tau\, p_\pi).
\]

\begin{figure}[t]
\centering
\begin{tikzpicture}[
  cell/.style={draw, minimum width=2.8cm, minimum height=0.6cm, anchor=center, font=\small},
  signal/.style={fill=maroon!25},
  weak/.style={fill=maroon!5},
  zero/.style={fill=gray!10},
]
  \node[anchor=south] at (1.4, 1.7) {state $x_1 \in \mathcal{G}_{h,\tau}$};
  \node[anchor=south, font=\footnotesize] at (1.4, 1.4) {(prob.\ $p_\pi$)};
  \node[anchor=south] at (4.5, 1.7) {state $x_2 \in \mathcal{G}_{h,\tau}^c$};
  \node[anchor=south, font=\footnotesize] at (4.5, 1.4) {(prob.\ $1 - p_\pi$)};
  \node[anchor=east, font=\small] at (-0.1, 1) {action $y_0$};
  \node[anchor=east, font=\small] at (-0.1, 0.4) {action $y_1$};
  \node[anchor=east, font=\small] at (-0.1, -0.2) {$y_k$, $k \geq 2$};
  \node[cell, zero] at (1.4, 1) {$R_{\max}/2$};
  \node[cell, signal] at (1.4, 0.4) {$R_{\max}/2 + \tau$};
  \node[cell, zero] at (1.4, -0.2) {$R_{\max}/2$};
  \node[cell, zero] at (4.5, 1) {$R_{\max}/2$};
  \node[cell, weak] at (4.5, 0.4) {$R_{\max}/2 + \varepsilon$};
  \node[cell, zero] at (4.5, -0.2) {$R_{\max}/2$};
\end{tikzpicture}
\caption{Reward structure of the $H = 1$ MDP used in
Proposition~\ref{prop:tight}. The base policy plays $y_0$ everywhere,
so $A^\pi(x, y_0) = 0$. Only action $y_1$ on $\mathcal{G}_{h,\tau}$
(maroon shading) carries the $\tau$-advantage signal; all other
cells contribute either the baseline (zero advantage) or a negligible
$\varepsilon \ll \tau p_\pi / (1 - p_\pi)$.}
\label{fig:tightness-construction}
\end{figure}

\begin{proposition}[Tightness of the random-reset estimator]
\label{prop:tight}
Consider the $H = 1$ construction above with greedy policy
$\pi^+ \equiv y_1$. Let $(x_i, y_i)_{i=1}^n$ be iid samples with
$x_i \sim \mu$ and $y_i \sim \mathrm{Unif}(\mathcal{Y})$, and define
\[
  \hat{\mathbb{A}} \;\coloneqq\; \frac{1}{n}\sum_{i=1}^n Y_i,
  \qquad
  Y_i \;\coloneqq\; |\mathcal{Y}|\bigl(\pi^+(y_i \mid x_i) - \pi(y_i \mid x_i)\bigr)\, r(x_i, y_i),
  \qquad
  \sigma^2 \;\coloneqq\; \mathrm{Var}(Y_i).
\]
Then $\mathbb{E}[\hat{\mathbb{A}}] = \mathbb{A}_{\pi,\mu}(\pi^+) = \Theta(\tau\, p_\pi)$
and $\sigma^2 = \Theta(|\mathcal{Y}|\, R_{\max}^2)$. There exist
absolute constants $c_1 > 0$ and $C_3 \geq 1$ such that, for every
sample size $n$ satisfying
\[
  C_3\, |\mathcal{Y}|^2 \;\leq\; n \;\leq\; c_1\, \frac{|\mathcal{Y}|\, R_{\max}^2}{\tau^2\, p_\pi^{\,2}},
\]
\[
  \Pr\!\bigl(\hat{\mathbb{A}} \leq 0\bigr) \;\geq\; 0.18.
\]
On the event $\{\hat{\mathbb{A}} \leq 0\}$, the \CPIRR{} step-size
formula $\hat\alpha = \min\{1, \hat{\mathbb{A}}/(H^2 R_{\max})\}$ does
not produce a valid mixture coefficient, as it returns a non-positive
value.
\end{proposition}

\begin{proof}
For the random sampler, each sample $(x_i, y_i)$ is drawn
independently with $x_i \sim \mu$ (the initial distribution) and
$y_i \sim \mathrm{Unif}(\mathcal{Y})$. Since $H = 1$ and rewards are
deterministic, $Q^\pi(x_i, y_i) = r(x_i, y_i)$. The empirical
advantage estimator is
\[
  \hat{\mathbb{A}} \;=\; \frac{1}{n}\sum_{i=1}^n Y_i,
  \qquad
  Y_i \;=\; |\mathcal{Y}|\bigl(\pi^+(y_i \mid x_i) - \pi(y_i \mid x_i)\bigr)\, r(x_i, y_i).
\]
With $\pi^+ \equiv y_1$ and $\pi \equiv y_0$, $Y_i$ takes three
values, depending on which action $y_i$ draws:
\[
  Y_i =
  \begin{cases}
    -|\mathcal{Y}|\, R_{\max}/2 & y_i = y_0, \\
    +|\mathcal{Y}|\, \bigl(R_{\max}/2 + A^\pi(x_i, y_1)\bigr) & y_i = y_1, \\
    0 & y_i \in \{y_2, \ldots, y_{|\mathcal{Y}|-1}\},
  \end{cases}
\]
where $A^\pi(x_i, y_1) = \tau$ if $x_i = x_1$ and $\varepsilon$ if
$x_i = x_2$, as defined in the resulting structure above.

\smallskip
\noindent\emph{Moments.} Conditional on $x_i$,
\[
  \mathbb{E}[Y_i \mid x_i] = A^\pi(x_i, y_1),
  \qquad
  \mathbb{E}[Y_i^2 \mid x_i]
  = |\mathcal{Y}|\, (R_{\max}/2)^2 + |\mathcal{Y}|\, \bigl(R_{\max}/2 + A^\pi(x_i, y_1)\bigr)^2.
\]
Averaging over $x_i \sim \mu$ and using $0 \leq A^\pi(x_i, y_1) \leq \tau \leq R_{\max}/2$,
\[
  \mathbb{A}_{\pi,\mu}(\pi^+) \;=\; p_\pi \tau + (1 - p_\pi) \varepsilon \;=\; \Theta(\tau\, p_\pi),
  \qquad
  \sigma^2 \;=\; \Theta(|\mathcal{Y}|\, R_{\max}^2),
\]
with $\sigma^2 \geq |\mathcal{Y}|\, R_{\max}^2/4$ (after absorbing
$\tau^2 \leq R_{\max}^2/4$ into the constant).

\smallskip
\noindent\emph{Anti-concentration (Berry--Esseen).} Set
$B \coloneqq 2 |\mathcal{Y}| R_{\max}$. Since $|Y_i - \mathbb{A}_{\pi,\mu}(\pi^+)| \leq B$, the
third absolute moment satisfies
$\rho = \mathbb{E}|Y_i - \mathbb{A}_{\pi,\mu}(\pi^+)|^3 \leq B^2 \sigma$.
Applying Lemma~\ref{lem:berry-esseen} to $\{Y_i\}_{i=1}^n$,
\[
  \sup_t \Bigl|\Pr\!\bigl((\hat{\mathbb{A}} - \mathbb{A}_{\pi,\mu}(\pi^+))/(\sigma/\sqrt{n}) \leq t\bigr) - \Phi(t)\Bigr|
  \;\leq\; \frac{0.6\, \rho}{\sigma^3 \sqrt{n}}
  \;\leq\; \frac{0.6\, B^2}{\sigma^2 \sqrt{n}}
  \;=\; \frac{2.4\, |\mathcal{Y}|^2 R_{\max}^2}{\sigma^2 \sqrt{n}}
  \;\leq\; \frac{10\, |\mathcal{Y}|}{\sqrt{n}}
  \;\leq\; \tfrac{1}{8}
\]
for $n \geq C_3\, |\mathcal{Y}|^2$, with $C_3 = 80^2$. The Berry-Esseen
bound then yields anti-concentration via the lower-tail estimate of
$\Phi$: in the regime
$C_3\, |\mathcal{Y}|^2 \leq n \leq \sigma^2/(4\, \mathbb{A}_{\pi,\mu}(\pi^+)^2)$,
the rescaled threshold
$\mathbb{A}_{\pi,\mu}(\pi^+) \sqrt{n}/\sigma \leq 1/2$, so
\[
  \Pr\!\bigl(\hat{\mathbb{A}} \leq 0\bigr)
  \;\geq\; \Phi(-1/2) - \tfrac{1}{8}
  \;\geq\; 0.18.
\]

\smallskip
\noindent Substituting $\mathbb{A}_{\pi,\mu}(\pi^+) = \Theta(\tau p_\pi)$
and $\sigma^2 = \Theta(|\mathcal{Y}|\, R_{\max}^2)$, the regime
$n \leq \sigma^2/(4\, \mathbb{A}_{\pi,\mu}(\pi^+)^2)$ takes the
explicit form $n \leq c_1\, |\mathcal{Y}|\, R_{\max}^2/(\tau^2 p_\pi^{\,2})$
stated in the proposition.
\end{proof}

\noindent The lower bound matches the $1/p_\pi^{\,2}$ factor of
Theorem~\ref{thm:main}'s upper bound for \CPIRR{}, so the credit
oracle's $1/p_\pi^{\,2}$ saving in Theorem~\ref{thm:main} is a
genuine algorithmic improvement, not an artifact of loose analysis.

\paragraph{Interpretation.}
The credit-assignment oracle eliminates the $p_\pi^{\,2}$ factor from
the sample count \emph{and} amplifies the per-iteration improvement
from $\tau^2 p_\pi^{\,2}$ to $\tau^2 p_\pi$. Both gains arise from
two distinct mechanisms: (1)~the larger signal estimated on
$\mathcal{G}_{h,\tau}$, formalized in the signal-to-noise reading above,
and (2)~the credit-aware simulation lemma~\ref{lem:credit-sim}, which
permits a $1/p_\pi$-larger step size in
Corollary~\ref{cor:credit-cpi}.

\section{Implementing the Credit Sampler via Rejection Sampling}
\label{app:rejection-sampling}

Algorithm~\ref{alg:rejection} formalizes the rejection sampler that
implements $\textsc{CreditSampler}(\pi, \mu)$ from the credit-assignment
oracle $\mathcal{O}$.

\begin{algorithm}[H]
\caption{Rejection-sampling implementation of $\textsc{CreditSampler}(\pi, \mu)$}
\label{alg:rejection}
\begin{algorithmic}[1]
\Require policy $\pi$, initial distribution $\mu$, oracle $\mathcal{O}$, target count $n$
\State $\mathcal{D} \gets \emptyset$
\While{$|\mathcal{D}| < n$}
  \State sample $h \sim \mathrm{Unif}[H]$ and $x \sim d_\pi^h$ via one reset into $\pi$'s trajectory distribution
  \State if $\mathcal{O}(x, h) = 1$ then $\mathcal{D} \gets \mathcal{D} \cup \{(x, h)\}$
\EndWhile
\State \Return $\mathcal{D}$
\end{algorithmic}
\end{algorithm}

\paragraph{Correctness.}
Each $(x, h)$ drawn in line~3 is on-policy; conditioning on the accept
event $\mathcal{O}(x, h) = 1$ returns a sample from the on-policy
distribution restricted to $\mathcal{G}_{h,\tau}$. The accepted samples are
therefore i.i.d.\ from the target distribution of
$\textsc{CreditSampler}(\pi, \mu)$.

\paragraph{Cost.}
Each trial succeeds with probability $p_\pi$, so $n$ accepted samples
require $n / p_\pi$ trials in expectation. A trial consists of one
reset plus one oracle query, both cheap. Expensive Q-rollouts
(Phase~1 of Algorithm~\ref{alg:cpi-oracle}) are executed only on
accepted samples, and thus incur cost $O(n)$:
\begin{center}
\renewcommand{\arraystretch}{1.2}
\begin{tabular}{ll}
Cheap calls (reset + oracle query): & $O(n / p_\pi)$ \\
Q-rollouts: & $O(n)$ \\
\end{tabular}
\end{center}
The Q-rollout count is independent of $p_\pi$; only the rejection step
scales as $1/p_\pi$.

\section{Training Details}
\label{app:training_details}

\paragraph{LoRA fine-tuning.}
All methods in the main results are trained with LoRA adapters
\citep{hu2021lora} (rank $64$, $\alpha = 64$); only the adapter
parameters are updated, leaving the base model frozen.

\paragraph{No clipped surrogate.}
We do not use the PPO-style clipped surrogate that standard GRPO
inherits; we find suffix gradient concentration (from
group-relative normalization with prefix masking) sufficiently
stable to drive learning without it. Table~\ref{tab:ablation_clip}
compares SRPO with and without the clipped surrogate on both base
models: clipping does not help on average and underperforms on
14/20 benchmark cells (winning only 4/20, with 2 ties).

\begin{table*}[htbp]
\centering
\scriptsize
\setlength{\tabcolsep}{2pt}
\caption{Clipped-surrogate ablation: SRPO trained on NuminaMath-Olympiads, with and without the PPO-style clipped ratio. Per-token loss aggregation is suffix-mean$\rightarrow$batch-mean in both rows; the only difference is the clip. Single seed (s42), no SD.}
\label{tab:ablation_clip}
\begin{tabular}{l*{10}{c}}
\toprule
Method & oly & hmmt & lvl5 & stra & ace & chem & bio & csqa & mat & phys \\
\midrule
\multicolumn{11}{l}{\textit{Qwen2.5-14B-Instruct}} \\
\midrule
SRPO w/o clip & 24.4 & 10.0 & \textbf{54.6} & \textbf{75.0} & 45.8 & \textbf{29.2} & \textbf{37.0} & \textbf{79.4} & \textbf{43.0} & \textbf{59.8} \\
SRPO w/ clip  & \textbf{26.2} & \textbf{13.3} & 53.4 & 53.0 & \textbf{46.2} & 24.0 & 27.0 & 79.2 & 28.8 & 31.4 \\
\midrule
\multicolumn{11}{l}{\textit{OLMo-3-7B-Instruct}} \\
\midrule
SRPO w/o clip & \textbf{25.0} & \textbf{16.7} & \textbf{68.2} & \textbf{69.6} & \textbf{60.2} & \textbf{22.8} & 28.4 & \textbf{73.4} & \textbf{55.0} & \textbf{52.0} \\
SRPO w/ clip  & \textbf{25.0} & 13.3 & 66.2 & 66.6 & 58.8 & \textbf{22.8} & \textbf{32.8} & 70.8 & 46.4 & 48.4 \\
\bottomrule
\end{tabular}
\end{table*}

\paragraph{Hyperparameters.}
Table~\ref{tab:hyperparameters} lists the optimization, rollout,
loss, and \SRPO{}/\RRPO{} localization settings used across all main
results. Values are taken from the trainer log dumps and the
launcher (\texttt{batch\_scripts/submit\_cpio2.sh}); we use the
same values for both base models, with parallelism scaled to model
size.

\begin{table}[htbp]
\centering
\scriptsize
\setlength{\tabcolsep}{4pt}
\caption{Training hyperparameters. Identical across base models
unless noted.}
\label{tab:hyperparameters}
\begin{tabular}{ll}
\toprule
\multicolumn{2}{l}{\textit{Optimization}} \\
Optimizer & AdamW \\
Learning rate & $5\!\times\!10^{-5}$ (constant) \\
LR warmup & none \\
Adam $(\beta_1, \beta_2)$ & $(0.9, 0.999)$ \\
Weight decay & $0.01$ \\
Gradient clipping ($\ell_2$) & $1.0$ \\
LoRA rank / $\alpha$ & $64\,/\,64$ \\
\midrule
\multicolumn{2}{l}{\textit{Schedule}} \\
Epochs & $2$ \\
Train batch size (prompts) & $32$ \\
PPO mini-batch size & $8$ \\
PPO epochs per step & $1$ \\
Seeds & $\{0, 42, 420\}$ \\
\midrule
\multicolumn{2}{l}{\textit{Rollout / sampling}} \\
Rollouts per prompt $G$ & $8$ \\
Temperature & $0.7$ \\
Top-$p$ & $0.95$ \\
Top-$k$ & off \\
Max prompt length & $2048$ \\
Max response length & $4096$ \\
Max thoughts per chain & $20$ \\
Max tokens per thought & $256$ \\
\midrule
\multicolumn{2}{l}{\textit{Hardware / parallelism}} \\
Qwen2.5-14B-Instruct & $4\times$ A100 40\,GB, TP$=4$ \\
OLMo-3-7B-Instruct & $2\times$ A100 40\,GB, TP$=2$ \\
Rollout backend & vLLM \\
Trainer parallelism & FSDP \\
\bottomrule
\end{tabular}
\end{table}

\section{Prompts}
\label{app:prompts}

\paragraph{Thought-MDP rollout prompt.}
The base agent loop wraps each problem with the template below before
sampling thoughts step-by-step. Each thought is generated as a
separate vLLM call with stop token \texttt{</thought>}; sampling
continues until the model emits \texttt{\textbackslash boxed\{answer\}}
or hits the chain cap.

\begin{promptpanel}[title=Thought-MDP rollout, verbatim, fontupper=\scriptsize\ttfamily]
You are solving a problem by producing one reasoning step at a time.

Do not try to solve the entire problem at once. Given the previously
taken steps, think about what the single next step should be, then
articulate it clearly and conclude just that step with </thought>.

Each step should be a complete, self-contained thought -- one
observation, calculation, or deduction that:
- Makes forward progress toward the solution
- Contains substantive reasoning (not filler like "let me think" or
  restating the problem)
- Coheres logically with the previous steps

When your next step arrives at the final answer, include
\boxed{answer} and end with </thought>.

Q: {question}
\end{promptpanel}

\paragraph{L2 self-localization prompt.}
After a failed Group~1 rollout, the trained policy is queried with
the prompt below to identify the originating error step. The chain
is rendered as \texttt{Step 1: \dots / Step 2: \dots / \dots} with
\texttt{</thought>} delimiters stripped. Picks come back as
\texttt{\textbackslash boxed\{N\}} and 98\% of OLMo-3-7B responses
are exactly that token sequence (Section~\ref{app:efficiency}).

\begin{promptpanel}[title=L2 self-localization, verbatim, fontupper=\scriptsize\ttfamily]
You are tasked with localizing the first erroneous thought in your
previous solution to this problem.

Problem: {question}

Your incorrect reasoning chain:
{Step 1: ...
Step 2: ...
...}

The final answer this chain produces is incorrect -- therefore at
least one step contains an error. The error you are looking for is
the originating step where a key decision or action derailed the
reasoning, not just the step where the failure ultimately becomes
visible. A misread of the problem, an unjustified assumption, or a
logical flaw can look fine for several follow-on steps before it
surfaces in the wrong answer. A step is erroneous if you cannot
justify its claims from the problem statement and earlier verified
steps alone. Find the originating step, not just the symptom.

Do NOT re-solve the problem. Your ONLY task is to identify the step
number of that originating error.

Requirements:
- Commit to exactly ONE step number (1 to ${n_\text{steps}}$).
- Stop at the first step you cannot justify.
- MANDATORY final line: your response MUST end with \boxed{N} on its
  own line, where N is the step index (1-indexed) of the first
  erroneous step in the chain above -- NOT the answer to the problem.
  Do NOT add any text after the \boxed{N}.
\end{promptpanel}

\section{Self-Localization Quality: Raw Deviation}
\label{app:loc_quality_raw}

Figure~\ref{fig:loc_quality_lcb_medium_raw} reproduces the four-panel
self-localization analysis from Section~\ref{sec:livecodebench} using the
\emph{raw} step-index deviation between the model's self-localization and the
Opus oracle, without any meaningful-step filtering. The eff5 version in the
main text (Figure~\ref{fig:loc_quality_lcb_medium}) collapses scaffolding
steps with fewer than 5 content tokens; the trends in (b)--(d) are
qualitatively the same under both definitions.

\begin{figure*}[t]
\centering
\includegraphics[width=\textwidth]{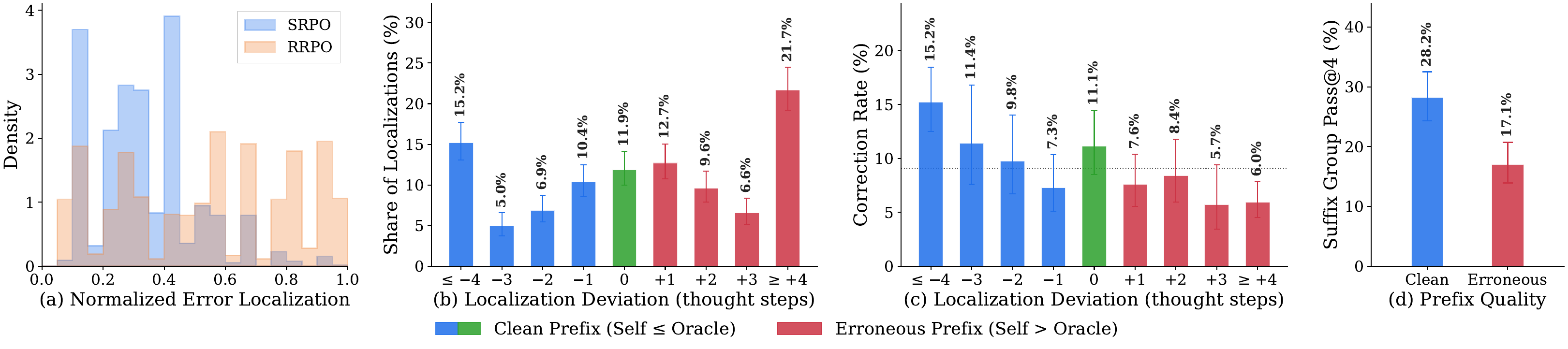}
\caption{Self-localization quality on LiveCodeBench v6 (medium), using the
raw step-index deviation between \SRPO{}'s self-localization and the Opus
oracle. Panels match Figure~\ref{fig:loc_quality_lcb_medium} but without the
meaningful-step filter.}
\label{fig:loc_quality_lcb_medium_raw}
\end{figure*}

\section{Full Sampling-Strategy Ablation}
\label{app:sampling_full}

Table~\ref{tab:ablation_sampling_full} reports the full sampling-strategy
comparison from Section~\ref{sec:sampling_design} across both
Qwen2.5-14B-Instruct and OLMo-3-7B-Instruct base models. The 1$\times$4 split
wins on the majority of columns under both base models.

\begin{table*}[htbp]
\centering
\scriptsize
\setlength{\tabcolsep}{2pt}
\caption{Sampling-strategy comparison for \SRPO{} under a fixed 8-rollout-per-prompt budget across both base models.}
\label{tab:ablation_sampling_full}
\begin{tabular}{l*{10}{c}}
\toprule
Method & oly & hmmt & lvl5 & stra & ace & chem & bio & csqa & mat & phys \\
\midrule
\multicolumn{11}{l}{\textit{Qwen2.5-14B-Instruct}} \\
\midrule
\SRPO{} (1$\times$4) & 25.5\,$\pm$\,1.2 & \textbf{6.7\,$\pm$\,2.7} & \textbf{55.2\,$\pm$\,0.7} & \textbf{74.9\,$\pm$\,0.2} & \textbf{46.2\,$\pm$\,0.9} & 24.5\,$\pm$\,3.8 & \textbf{33.9\,$\pm$\,2.3} & \textbf{80.6\,$\pm$\,0.9} & \textbf{39.3\,$\pm$\,2.6} & \textbf{45.5\,$\pm$\,11.8} \\
\SRPO{} (2$\times$4) & \textbf{25.7\,$\pm$\,2.0} & 5.6\,$\pm$\,1.6 & 53.4\,$\pm$\,2.0 & 60.4\,$\pm$\,9.3 & 42.0\,$\pm$\,0.9 & \textbf{26.0\,$\pm$\,5.7} & 32.7\,$\pm$\,2.8 & 75.0\,$\pm$\,1.4 & 34.3\,$\pm$\,7.4 & 34.3\,$\pm$\,15.6 \\
\SRPO{} (1$\times$8) & 25.0\,$\pm$\,4.8 & 4.4\,$\pm$\,1.6 & 44.4\,$\pm$\,13.8 & 71.6\,$\pm$\,3.3 & 37.1\,$\pm$\,9.9 & 21.2\,$\pm$\,3.4 & 32.3\,$\pm$\,4.0 & 73.5\,$\pm$\,5.7 & 26.6\,$\pm$\,4.6 & 27.4\,$\pm$\,6.5 \\
\midrule
\multicolumn{11}{l}{\textit{OLMo-3-7B-Instruct}} \\
\midrule
\SRPO{} (1$\times$4) & 24.8\,$\pm$\,1.9 & \textbf{15.6\,$\pm$\,4.2} & \textbf{67.5\,$\pm$\,3.5} & \textbf{66.6\,$\pm$\,2.5} & \textbf{59.6\,$\pm$\,2.6} & \textbf{24.7\,$\pm$\,1.4} & 28.7\,$\pm$\,0.6 & \textbf{74.8\,$\pm$\,1.0} & \textbf{55.9\,$\pm$\,0.6} & \textbf{54.5\,$\pm$\,2.0} \\
\SRPO{} (2$\times$4) & \textbf{26.1\,$\pm$\,1.2} & 14.4\,$\pm$\,3.1 & 66.0\,$\pm$\,3.1 & 65.9\,$\pm$\,2.2 & 56.7\,$\pm$\,2.7 & 22.4\,$\pm$\,0.5 & 27.9\,$\pm$\,0.2 & 65.9\,$\pm$\,3.3 & 45.7\,$\pm$\,1.3 & 43.2\,$\pm$\,5.7 \\
\SRPO{} (1$\times$8) & 23.1\,$\pm$\,2.6 & 12.2\,$\pm$\,3.1 & 62.3\,$\pm$\,2.5 & 64.1\,$\pm$\,3.4 & 53.7\,$\pm$\,6.3 & 22.8\,$\pm$\,1.2 & \textbf{29.8\,$\pm$\,1.9} & 59.5\,$\pm$\,14.9 & 43.5\,$\pm$\,5.8 & 43.4\,$\pm$\,8.6 \\
\bottomrule
\end{tabular}
\end{table*}

\section{Per-Token Gradient Concentration}
\label{app:grad_concentration}

Figure~\ref{fig:grad_tree_p2} visualizes one SRPO update on a single
prompt: four shared-prefix rollouts on top, four base rollouts on
bottom. Each rollout is laid out as a sequence of thought blocks
(Section~\ref{sec:thought-mdp}); gray blocks mark the shared prefix
that is masked out of the shared-prefix loss, so no gradient lands on
those tokens. This appendix derives the per-token gradient signal
$g_{i,t}$ used in Section~\ref{sec:livecodebench}, formalizes
the per-thought aggregation that determines each block's color, and
shows how $g_{i,t}$ folds into the gradient of the full SRPO loss.

\paragraph{Derivation.} For a shared-prefix rollout $i$ with suffix tokens
$y_{i,1}, \ldots, y_{i,T_i}$ sampled under reset state $x^*$, the
per-token contribution to the loss from Section~\ref{sec:loss} is
\[
  L_t \;=\; -\frac{\hat A_i}{T_i} \log \pi_\theta\bigl(y_{i,t} \,\big|\, x^*, y_{i,<t}\bigr).
\]
Writing the policy as a softmax over the actor's logits $z = (z_v)_v$,
$\pi_\theta(y_{i,t} \mid \cdot) = \exp(z_{y_{i,t}}) / \sum_v \exp(z_v)$, we have
\[
  \log \pi_\theta(y_{i,t} \mid \cdot) \;=\; z_{y_{i,t}} \,-\, \log \sum_v \exp(z_v),
  \qquad
  \frac{\partial \log \pi_\theta(y_{i,t} \mid \cdot)}{\partial z_{y_{i,t}}} \;=\; 1 - \pi_\theta(y_{i,t} \mid \cdot).
\]
Chaining,
\[
  \frac{\partial L_t}{\partial z_{y_{i,t}}}
  \;=\;
  -\frac{\hat A_i}{T_i} \bigl(1 - \pi_\theta(y_{i,t} \mid \cdot)\bigr),
\]
and since $1 - \pi_\theta(y_{i,t} \mid \cdot) \in [0, 1)$, the
magnitude is $g_{i,t} = (|\hat A_i|/T_i)(1 - \pi_\theta(y_{i,t} \mid \cdot))$ as used in Section~\ref{sec:livecodebench}. The
first factor $|\hat A_i|/T_i$ is constant in $t$ across the rollout's
active region; the second factor is the source of within-trajectory
variation. The base group uses the same expression with the full
response as the active region.

\paragraph{Folding into the full loss.} $L_t$ is one summand of the
shared-prefix loss
$\mathcal{L}_{\mathrm{SP}} = \frac{1}{G}\sum_{i=1}^{G} \sum_{t=1}^{T_i} L_t$,
and $\mathcal{L}_{\mathrm{base}}$ has the analogous form over full
responses. The logit $z_{y_{i,t}}$ appears in only one summand, so
the corresponding entry of the gradient is
\[
  \frac{\partial \mathcal{L}_{\mathrm{SP}}}{\partial z_{y_{i,t}}}
  \;=\;
  \frac{1}{G}\,\frac{\partial L_t}{\partial z_{y_{i,t}}}
  \;=\;
  -\frac{1}{G}\cdot\frac{\hat A_i}{T_i}\bigl(1 - \pi_\theta(y_{i,t} \mid \cdot)\bigr),
\]
identical for $\mathcal{L}_{\mathrm{base}}$ at $T_i$ equal to the
full response length. $g_{i,t}$ is therefore the magnitude of the
$(i,t)$ entry of $\nabla_z \mathcal{L}$ that the optimizer applies at
this step, up to the global rescaling $1/G$ shared by all entries;
Figure~\ref{fig:grad_tree_p2} plots these entries directly.

\paragraph{Per-thought aggregation.} A rollout's active tokens
partition into thoughts (Section~\ref{sec:thought-mdp}). The color of
thought block $(i, h)$ in Figure~\ref{fig:grad_tree_p2} is the mean
of $g_{i,t}$ over the thought's active tokens. Masked thoughts
(a shared-prefix rollout's prefix) have no active tokens and are drawn gray. For each prompt $p$, we also summarize $g_{i,t}$ at the group
level by first averaging over each rollout's active tokens,
$\bar g_i = \frac{1}{T_i}\sum_{t=1}^{T_i} g_{i,t}$, and then over
the rollouts in that group with nonzero advantage:
\[
  \overline g_{\mathrm{base}}(p)
  \;=\;
  \frac{1}{n_{\mathrm{base}}^*(p)} \sum_{i \in \mathrm{base},\, |\hat A_i| > 0} \bar g_i,
  \qquad
  \overline g_{\mathrm{SP}}(p)
  \;=\;
  \frac{1}{n_{\mathrm{SP}}^*(p)} \sum_{i \in \mathrm{SP},\, |\hat A_i| > 0} \bar g_i.
\]
A rollout with $|\hat A_i| = 0$ contributes no gradient and is
excluded.

\paragraph{Effect of the prefix mask.} Because the first factor
$|\hat A_i|/T_i$ is constant in $t$ across a rollout's active region,
it sets the per-rollout floor on $g_{i,t}$. For a shared-prefix rollout, the
prefix mask shrinks $T_i$ from the rollout's full-response length to
its suffix length, raising that floor on every active (suffix) token
of the rollout. The base loss, with no masking, retains the full
$T_i$ and spreads each rollout's credit across its entire response.
We verify this empirically throughout training in
§\ref{app:grad_concentration_arc}.

\subsection{Gradient signal concentration throughout training}
\label{app:grad_concentration_arc}

In this LiveCodeBench training run over 1 epoch, we observe that (i)
when both rollout groups produce gradient signal on the same prompt,
shared-prefix rollouts receive a higher per-token gradient signal
than base rollouts; and (ii) the population of signal-bearing
shared-prefix groups grows
across training as the model's self-correction skill improves.

We reproduce the LiveCodeBench-medium ep1 schedule on
OLMo-3-7B-Instruct (11 update steps, 32 prompts per step, 8 rollouts
per prompt) and report $\bar g_i$ averaged within each group on each
prompt, then across prompts at each step.

When both groups produce gradient signal on the same prompt,
shared-prefix rollouts receive a higher per-token gradient signal
than base rollouts: $\bar g^{\mathrm{SP}}/\bar g^{\mathrm{base}} > 1$ at 10 of
the 11 training steps, range $0.97$--$1.76$, typical value $1.21$
(Figure~\ref{fig:grad_signal_trajectory}, left). The conditioning
isolates the architectural claim of
§\ref{app:grad_concentration}: group-relative advantages $\hat A_i =
(r_i - \bar r)/\sigma_r$ collapse to zero whenever all four rollouts
in a (prompt, group) pair land at the same verifier outcome, so the
subset where each group delivers at least one nonzero-advantage
rollout is the population on which masking can act.

Shared-prefix groups produce these dual-signal cases less often than
base groups (all-same-outcome rates $65$--$88\%$ for shared-prefix
versus $31$--$66\%$ for base), which is structural: shared-prefix
rollouts continue from a prefix the model has flagged as incorrect,
while base rollouts are fresh iid samples from $x_0$, so all-fail is
the modal shared-prefix outcome. As the policy's self-correction
skill improves across training, the shared-prefix per-rollout pass
rate rises from $3.9\%$ at step 1 to $12.5\%$ at step 11 (Pearson
$r=+0.81$ vs.\ step; Figure~\ref{fig:grad_signal_outcomes}, right),
reducing the shared-prefix all-same-outcome rate from $87.5\%$ to
$65.6\%$. The fraction of
prompts where the architectural concentration acts therefore grows
across training.

\begin{figure}[h]
\centering
\begin{minipage}[c]{0.48\linewidth}
  \centering
  \includegraphics[width=\linewidth]{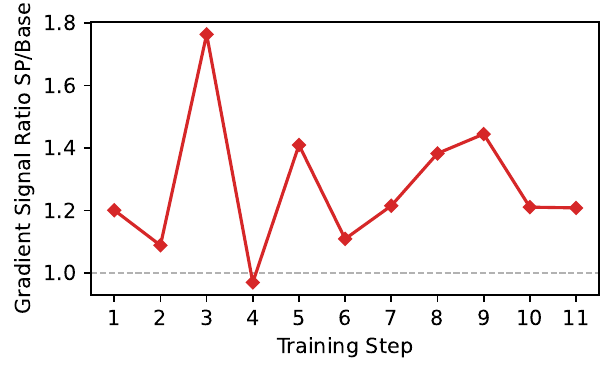}
\end{minipage}\hfill
\begin{minipage}[c]{0.48\linewidth}
  \centering
  \includegraphics[width=\linewidth]{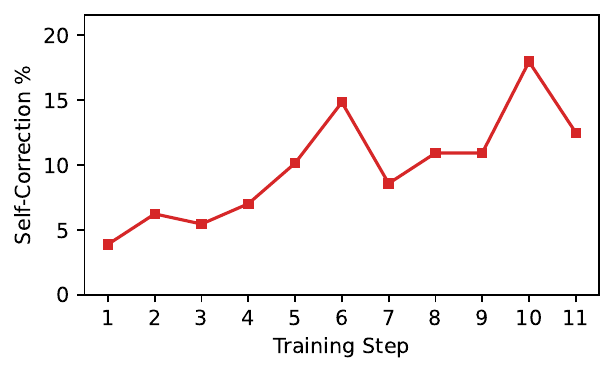}
\end{minipage}
\caption{\textbf{Left:} per-token gradient signal ratio
$\bar g^{\mathrm{SP}}/\bar g^{\mathrm{base}}$ throughout training on
prompts where both groups deliver gradient (each group has at least
one rollout with $|\hat A|>0$). The ratio exceeds 1 at 10 of 11 steps:
the prefix mask is concentrating per-token signal onto the suffix.
\textbf{Right:} shared-prefix per-rollout pass rate across training
($3.9\% \to 12.5\%$, Pearson $r=+0.81$ vs.\ step). As self-correction
improves, the shared-prefix group degeneracy rate shrinks from
$87.5\%$ to $65.6\%$, growing the population of prompts on which the
left-panel concentration acts. Shared-prefix rollouts are conditional
on repairing a flagged-error parent prefix, so this curve measures
self-correction specifically.}
\label{fig:grad_signal_trajectory}
\label{fig:grad_signal_outcomes}
\end{figure}

\section{Training Efficiency}
\label{app:efficiency}

We measure per-method compute on LiveCodeBench-Medium with
OLMo-3-7B-Instruct (seed~42, one epoch $=$ 11 training steps, 8
rollouts per prompt, batch size 32, $2\times$ A100 40\,GB per run).
Step timings come from the trainer log: \texttt{timing\_s/step} is
the training portion of a step (generation, advantages, log
probabilities, actor update); \texttt{timing\_s/gen} is the
generation portion alone; \texttt{timing\_s/testing} is the
validation pass at each step and is method-independent. Total wall
clock through ep1 is the sum of \texttt{timing\_s/step} and
\texttt{timing\_s/testing} over steps 1--11.
Table~\ref{tab:efficiency_abs} reports raw values;
Table~\ref{tab:efficiency_rel} normalizes to GRPO.

\begin{table}[htbp]
\centering
\scriptsize
\setlength{\tabcolsep}{4pt}
\caption{Per-method compute on lcbm through ep1 (OLMo-3-7B-Instruct,
seed~42, 11 steps, $2\times$~A100~40\,GB). \emph{train\_h} is training-only
wall clock; \emph{gen\_h} is rollout generation alone;
\emph{val\_h} is the (method-independent) validation pass time;
\emph{total\_h} = \emph{train\_h} + \emph{val\_h}; \emph{tokens} is
$\sum$\texttt{perf/total\_num\_tokens} (all forward-passed tokens
in training steps); \emph{resp$_\mu$} is mean response length
across rollouts; \emph{gen tok/s} is \emph{tokens}/\emph{gen\_h}.}
\label{tab:efficiency_abs}
\begin{tabular}{lrrrrrrr}
\toprule
Method & train\_h & gen\_h & val\_h & total\_h & tokens & resp$_\mu$ & gen tok/s \\
\midrule
GRPO & 3.44 & 2.70 & 4.32 &  7.75 & 8.10\,M & 2281 & 833 \\
\RRPO{} & 5.33 & 4.48 & 6.50 & 11.83 & 9.94\,M & 2932 & 617 \\
\SRPO{} & 5.21 & 4.36 & 4.93 & 10.14 & 8.75\,M & 2510 & 557 \\
\bottomrule
\end{tabular}
\end{table}

\begin{table}[htbp]
\centering
\scriptsize
\setlength{\tabcolsep}{4pt}
\caption{Same quantities normalized to GRPO ($=1.00$).}
\label{tab:efficiency_rel}
\begin{tabular}{lrrrrr}
\toprule
Method & train & gen & total & tokens & resp$_\mu$ \\
\midrule
GRPO & 1.00 & 1.00 & 1.00 & 1.00 & 1.00 \\
\RRPO{} & 1.55 & 1.66 & 1.53 & 1.23 & 1.29 \\
\SRPO{} & 1.51 & 1.61 & 1.31 & 1.08 & 1.10 \\
\bottomrule
\end{tabular}
\end{table}

\paragraph{Discussion.}
\SRPO{} emits $1.08\times$ as many tokens as GRPO through ep1, and \RRPO{}
$1.23\times$, despite both running a second group of $G$ rollouts on
top of the base group: G2 shares the G1 prefix up to the reset step,
so only the suffix is autoregressively sampled rather than a full
fresh rollout (mean suffix length $\approx 1720$~tokens for \SRPO{} at
step~1). Training-only wall clock is $\sim 1.5\times$ GRPO
(Table~\ref{tab:efficiency_rel}), reflecting the serial dependency
of G2 on the self-localization step computed from G1.

\paragraph{Localization-call cost.}
The token counts above include only G1$+$G2 rollout tokens; the
self-localization call's output is discarded as non-training data
and is not in \texttt{perf/total\_num\_tokens}. For
OLMo-3-7B-Instruct on lcbm, this omission is negligible: $98\%$ of
localization responses are just \texttt{\textbackslash
boxed\{N\}\textless|endoftext|\textgreater} ($\approx 9$~tokens),
totalling $\sim 3\,$K tokens across the entire ep1
($0.04\%$ of training tokens). Models with more verbose
localization behavior could push this share materially higher, so
future runs should log per-call localization output length and add
it to the token total.

\end{document}